\documentclass[conference]{IEEEtran}
\IEEEoverridecommandlockouts
\usepackage{libertine}
\usepackage{enumitem}
\usepackage{algorithm}
\usepackage{algorithmicx}
\usepackage{algpseudocode} 
\usepackage{xcolor}
\usepackage{graphicx} 
\usepackage{epstopdf}
\usepackage{subfigure}
\usepackage{amsmath,bm}
\usepackage{multirow}
\usepackage{url}
\usepackage{amssymb}
\usepackage{bbding}
\usepackage{marginnote}
\usepackage{float}
\usepackage{booktabs}
\usepackage{amsthm}
\usepackage[numbers]{natbib}
\usepackage{balance}

\newtheorem{theorem}{Theorem}[section]

\usepackage[textsize=tiny]{todonotes}

\usepackage{pifont}

\usepackage{fontawesome5}  
\usepackage[hidelinks]{hyperref}  
\newcommand{\orcid}[1]{
    \href{https://orcid.org/#1}{\faIcon[regular]{orcid}}
}

\newcommand{\ICDER}{\color{blue}}

\def\model{\textbf{AutoHFormer}}

\usepackage{tikz}
\usepackage{booktabs}
\usepackage{pgfplots}
\usepackage{amsmath}
\usepackage{xcolor}
\usepackage{colortbl}

\renewcommand{\ICDER}{\color{black}}

\begin{document}

\title{AutoHFormer: Efficient Hierarchical Autoregressive Transformer for Time Series Prediction}

\author{Qianru Zhang\textsuperscript{\orcid{0000-0002-5843-6187}}\IEEEauthorrefmark{1}\IEEEauthorrefmark{6}, Honggang Wen\textsuperscript{\orcid{0009-0006-8691-8569}}\IEEEauthorrefmark{1}\IEEEauthorrefmark{6}\IEEEcompsocitemizethanks{\IEEEcompsocthanksitem\IEEEauthorrefmark{6}Equal Contribution.}, Ming Li\textsuperscript{\orcid{0000-0002-1218-2804}}\IEEEauthorrefmark{2}, Dong Huang\textsuperscript{\orcid{0000-0002-4275-3006}}\IEEEauthorrefmark{3}\IEEEauthorrefmark{4}, Siu-Ming Yiu\textsuperscript{\orcid{0000-0002-3975-8500}}\IEEEauthorrefmark{1}\IEEEauthorrefmark{4}\\
Christian S. Jensen\textsuperscript{\orcid{0000-0002-9697-7670}}\IEEEauthorrefmark{8},
Pietro Liò\textsuperscript{\orcid{0000-0002-0540-5053}}\IEEEauthorrefmark{5}\IEEEauthorrefmark{4}\IEEEcompsocitemizethanks{\IEEEcompsocthanksitem\IEEEauthorrefmark{4}Corresponding author.}\\
\IEEEauthorblockA{\IEEEauthorrefmark{1}School of Computing and Data Science, The University of Hong Kong (HKU)\\
\IEEEauthorrefmark{2}Zhejiang Key Laboratory of Intelligent Education Technology and Application, Zhejiang Normal University (ZJNU)\\
\IEEEauthorrefmark{3}Department of Computer Science, National University of Singapore (NUS)\\
\IEEEauthorrefmark{8}Department of Computer Science, Aalborg University (AU)\\
\IEEEauthorrefmark{5}Department of Computer Science and Technology, Cambridge University (Cambridge)\\
}
}

\maketitle

\pagenumbering{arabic}
\setcounter{page}{1}

\begin{abstract}
Time series forecasting requires architectures that simultaneously achieve three competing objectives: (1) strict temporal causality for reliable predictions, (2) sub-quadratic complexity for practical scalability, and (3) multi-scale pattern recognition for accurate long-horizon forecasting. We introduce \model, a hierarchical autoregressive transformer that addresses these challenges through three key innovations: \textbf{1) Hierarchical Temporal Modeling:} Our architecture decomposes predictions into segment-level blocks processed in parallel, followed by intra-segment sequential refinement. This dual-scale approach maintains temporal coherence while enabling efficient computation. \textbf{2) Dynamic Windowed Attention:} The attention mechanism employs learnable causal windows with exponential decay, reducing complexity while preserving precise temporal relationships. This design avoids both the anti-causal violations of standard transformers and the sequential bottlenecks of RNN hybrids. \textbf{3) Adaptive Temporal Encoding:} a novel position encoding system is adopted to capture time patterns at multiple scales. It combines fixed oscillating patterns for short-term variations with learnable decay rates for long-term trends. Comprehensive experiments demonstrate that \model\ 10.76× faster training and 6.06× memory reduction compared to PatchTST on PEMS08, while maintaining consistent accuracy across 96-720 step horizons in most of cases. These breakthroughs establish new benchmarks for efficient and precise time series modeling. Implementations of our method and all baselines in hierarchical autoregressive mechanism are available at {\color{blue}{\url{https://github.com/CoderPowerBeyond/AutoHFormer}}}.
\end{abstract}

\renewcommand{\thesection}{\Roman{section}}
\setcounter{section}{0}

\section{Introduction}
\label{sec:intro}

 \begin{figure*}[ht]
     \centering
     \includegraphics[width=\textwidth]{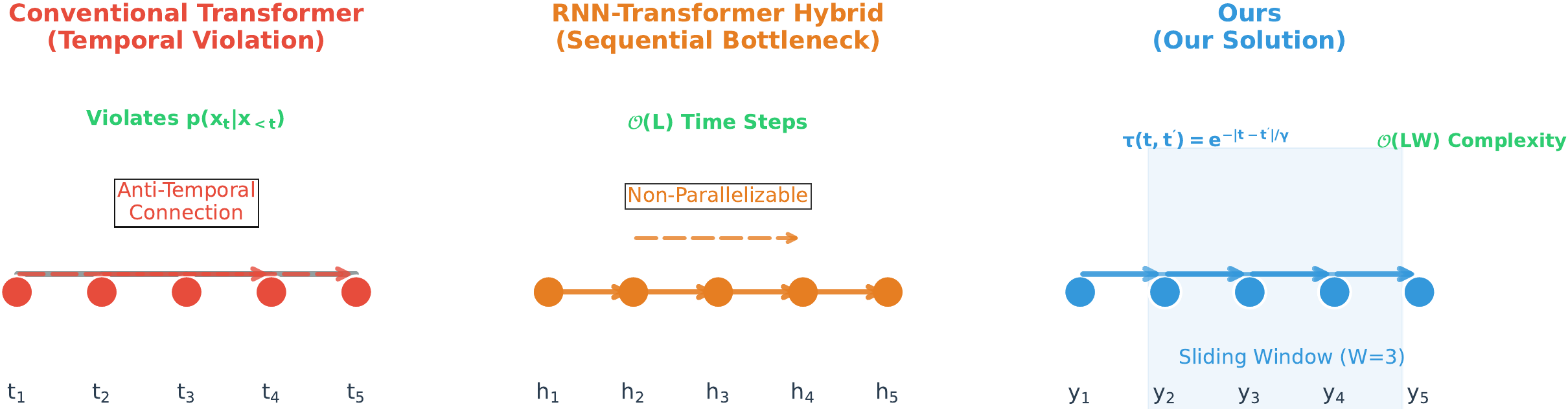}
     \caption{
     Architectural comparison of time series modeling approaches: 
    (a) \textbf{Conventional Transformer} suffers from anti-causal attention flows (red dashed arrows) that violate the fundamental autoregressive principle $p(x_t|x_{<t})$. 
    (b) \textbf{RNN-Transformer Hybrid} enforces causality through sequential processing (orange solid arrows) but introduces an $\mathcal{O}(L)$ computational bottleneck that prevents parallel training. 
     (c) \textbf{Ours} (our solution) combines: 
    \textcircled{1} Strictly causal attention within a sliding window $W$ (blue shaded region), 
    \textcircled{2} Exponentially decaying attention weights $\tau(t,t')=e^{-|t-t'|/\gamma}$ (visualized by arrow opacity gradient), and 
   \textcircled{3} $\mathcal{O}(LW)$ complexity through windowed parallel processing. 
     The thickness of blue arrows represents attention magnitude, demonstrating our model's ability to simultaneously maintain temporal causality while enabling efficient parallel computation.
    }
    \label{fig:attention_comparison}
 \end{figure*}

Time series forecasting~\cite{li2023revisiting,wang2024mamba,zhang2025fldmamba,huang2024long,wu2022timesnet} stands as a fundamental pillar of modern predictive analytics, enabling data-driven decision making across numerous mission-critical domains. As demonstrated in recent literature~\cite{yi2018integrated,zhangprompts}, this task has become increasingly vital in our data-rich era. In financial markets~\cite{greenwood1997financial,zhang2024survey}, accurate forecasting drives risk management strategies. Energy sector applications require precise demand predictions to optimize renewable integration. Transportation systems~\cite{yuan2021effective,zhang2025efficient,zhang2025survey} rely on temporal pattern analysis for traffic flow optimization. The importance of time-series analysis also extends to public safety through crime prediction~\cite{yi2018integrated} and model user behaviors with time evolving~\cite{zhang2025hmamba,zhang2025m2rec}. These diverse applications share a common need: the ability to extract meaningful patterns from temporal data to inform proactive decision-making.

{\ICDER{The evolution of time series forecasting has been significantly advanced by Transformer architectures~\cite{vaswani2017attention}, yet persistent limitations hinder their practical deployment. As shown in Figure~\ref{fig:attention_comparison}, conventional Transformers suffer from \textbf{anti-causal attention flows}, violating the fundamental $p(x_t|x_{<t})$ principle through bidirectional information propagation, which undermines autoregressive modeling. RNN-Transformer hybrids, while addressing causality, impose \textbf{sequential computation bottlenecks} ($\mathcal{O}(L)$ time steps, where $L$ is the sequence length), negating the parallel processing advantages of Transformers. Three fundamental shortcomings emerge from this architectural landscape: \textbf{(1) Quadratic Complexity}: Standard self-attention's $\mathcal{O}(L^2)$ memory and computation requirements make long-sequence processing infeasible; \textbf{(2) Error Propagation}: Autoregressive methods accumulate prediction errors exponentially across time steps; \textbf{(3) Temporal Representation Rigidity}: Fixed attention patterns fail to adapt to varying time scales, leading to suboptimal performance on multi-scale patterns.

\textbf{Recent Advances and Limitations:} Attention mechanisms have emerged as a powerful framework for temporal modeling, aiming to simultaneously optimize three key objectives: \textbf{computational efficiency}, \textbf{temporal coherence}, and \textbf{multi-scale pattern capture}. While recent innovations such as Informer~\cite{zhou2021informer}, Pyraformer~\cite{liu2021pyraformer}, and PatchTST~\cite{huang2024long} address individual aspects, they fail to achieve all three objectives simultaneously. Informer employs sparse attention for scalability but sacrifices multi-scale modeling; Pyraformer enforces hierarchical pyramidal structures that incur $O(L\log L)$ complexity, limiting adaptability to dynamic temporal scales; PatchTST efficiently handles long sequences but relies on anti-causal patching that violates strict causality. Similarly, GNN-based methods~\cite{chen2024graph,cirstea2021graph} model variable dependencies but struggle with non-stationary patterns and computational efficiency for long-range dependencies. These partial solutions leave the three core challenges—\textit{causality, efficiency, and multi-scale adaptability}—unresolved within a unified framework.

\textbf{Our Contribution:} Unlike existing approaches that address these challenges in isolation, our proposed method, \model, introduces a cohesive framework that integrates \textit{Dynamic Causal Windows} for efficient computation and strict temporal causality. By leveraging learnable receptive fields and position-aware exponential decay, \model\ adapts to varying time scales while maintaining computational efficiency with $O(LW)$ complexity. This mechanism ensures selective focus on relevant temporal patterns, enabling multi-scale adaptability from minute-level transients to daily periodicities. Furthermore, our hierarchical attention framework unifies \textit{top-down segment-level attention} with \textit{bottom-up step-level refinement}, balancing coarse-grained temporal trends with fine-grained details. Unlike prior methods, \model\ simultaneously achieves sub-quadratic scaling, strict causality, and robust long-horizon forecasting, addressing the three competing objectives in a unified manner. As validated by our empirical results across diverse benchmarks, \model\ delivers state-of-the-art performance with improved stability, reduced error accumulation, and no computational bottlenecks.
}}

We summarize our contributions as follows:
\begin{itemize}
    \item \textbf{A Unified and Efficient Hierarchical Autoregressive Framework for Time Series Forecasting.} We propose \model, a novel hierarchical autoregressive Transformer that integrates block-level (segment) processing with step-level refinement. This design balances coarse-grained temporal trends with fine-grained details, enabling efficient and accurate long-horizon forecasting.

    \item \textbf{Dynamic Causal Attention with Temporal Coherence.} We introduce a memory-efficient attention mechanism based on \textit{Dynamic Causal Windows}, which combines learnable receptive fields with position-aware exponential decay. This mechanism enforces strict causality while maintaining temporal coherence and achieving sub-quadratic complexity ($O(LW)$).

    \item \textbf{Adaptive Temporal Encoding for Multi-scale Patterns.} To address varying temporal scales, we propose a relative positional encoding method with learned decay coefficients. This approach captures both short transients and long-term multi-scale patterns, enabling robust performance across diverse datasets and scenarios.

    \item \textbf{Comprehensive Empirical Validation.} We conduct extensive experiments on benchmark datasets, demonstrating that \model\ outperforms 11 state-of-the-art baselines, including transformer variants (iTransformer, Informer, Autoformer) and modern architectures (PatchTST, TimeMixer). \model\ achieves particularly strong results in long-sequence forecasting (720+ prediction steps). All implementations, including hierarchical autoregressive mechanisms for our method and baselines, are available at {\color{blue}{\url{https://github.com/CoderPowerBeyond/AutoHFormer}}}.
\end{itemize}

\section{Method}
\label{sec:method}

\begin{figure*}
    \centering
    \includegraphics[width=0.98\textwidth, height=0.25\textheight]{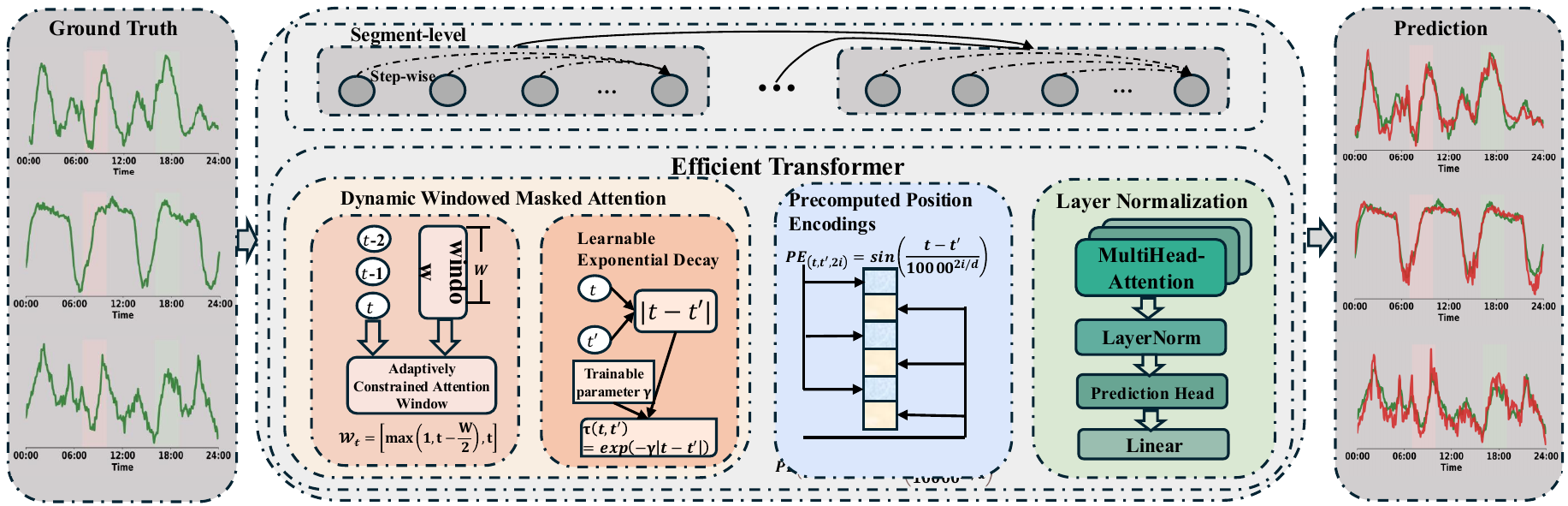}
    \vspace{-0.15in}
    \caption{The Overview of \model. The left part is the input time series data. The right part is the prediction. For the middle part, (a) the top component is hierarchical autoregressive mechanism including segment-level part and step-wise part. (b) For the bottom part of the middle, there are three key components: (1) a dynamic windowed attention mechanism that computes localized attention patterns with adaptive window sizing, (2) precomputed relative position encodings that capture temporal relationships through sinusoidal embeddings, and (3) a series of several sub-models (efficient transformer) with layer normalization that progressively refine feature representations while maintaining strict causality. The complete system transforms raw input sequences into accurate predictions through this optimized transformer-based pipeline, achieving both computational efficiency and modeling precision.}
    \label{fig:framework}
    \vspace{-0.2in}
\end{figure*}

\subsection{Overview of \model}
The \model\ framework introduces a novel hierarchical autoregressive transformer that fundamentally advances time series forecasting through three key innovations. \textbf{First}, our hierarchical architecture combines segment-wise processing with stepwise refinement to eliminate error propagation in long-horizon predictions. \textbf{Second}, the optimized windowed attention mechanism employs learnable causal windows $ \mathcal{W}_t = [t-w,t]$ ($W$: window size) to achieve $\mathcal{O}(LW)$ complexity while preserving temporal coherence ($p(x_t|x_{<t})$). The resulting system maintains sub-quadratic scaling ($LW \ll L^2$ for $L \gg W$) with full parallelizability, while the fixed per-token memory requirement ($\mathcal{O}(W)$) ensures hardware efficiency. \textbf{Third}, adaptive temporal position encodings capture multi-scale patterns and short transients through hybrid sinusoidal-decay embeddings. Theoretically grounded in Section~\ref{sec:theory} and empirically validated in Section~\ref{sec:exp}, \model\ demonstrates: (1) 10.76$\times$ faster training than conventional transformer-based PatchTST on PEMS08, (2) consistent accuracy across 96-720 step horizons in most of cases, and (3) 6.06$\times$ memory reduction than PatchTST on PEMS08 establishing new standards for efficient and accurate temporal modeling.

\subsection{Preliminary}\label{sec:pre}

\subsubsection{Autoregressive (AR) Mechanism.}
The AR($p$) model expresses the current value $X_t$ as a linear combination of its $p$ previous values plus a stochastic error term:
\begin{equation}
X_t = c + \phi_1 X_{t-1} + \phi_2 X_{t-2} + \cdots + \phi_p X_{t-p} + \epsilon_t
\label{eq:ar}
\end{equation}
Here $c$ denotes the constant term (unconditional mean when $\phi_i=0$). $\phi_i$ denotes autoregressive coefficients (weights for past values).  $\epsilon_t \sim \mathcal{N}(0,\sigma^2)$ denotes the white noise innovation.
The model captures how past values influence the present through the weights $\phi_i$. Higher absolute values of $\phi_i$ indicate stronger dependence on the $i$-th lag.

\subsubsection{Problem Definition}
Given an input sequence $\mathbf{X}_{1:L} \in \mathbb{R}^{L \times V}$ with $V$ variates, we predict $K$ consecutive segments $\{\widehat{\mathbf{Y}}_1,...,\widehat{\mathbf{Y}}_K\}$ where each segment $\widehat{\mathbf{Y}}_h \in \mathbb{R}^{H \times V}$ has fixed horizon $H$. The task is to learn two mapping functions $\mathcal{F}$ and $\mathcal{G}$ following a two-level autoregressive process:

\noindent\textbf{1) Segment-level Autoregression}:
    \begin{equation}
        \widehat{\mathbf{Y}}_h = \mathcal{F}\big(\mathbf{X}_{1:L}, \widehat{\mathbf{Y}}_{1:h-1}\big) \quad \forall h \in \{1,...,K\}
    \end{equation}
\textbf{2) Step-wise Generation} within each segment:
    \begin{equation}
    \begin{aligned}
        \widehat{y}_{(h-1)H + t} = \mathcal{G}\big(\mathbf{X}_{1:L}, 
        \widehat{\mathbf{Y}}_{1:h-1}, \widehat{y}_{(h-1)H+1},...,\\\widehat{y}_{(h-1)H+t-1}\big)
        \quad \forall t \in \{1,...,H\}
    \end{aligned}
    \end{equation}

The model architecture employs several key design parameters to govern its forecasting behavior. The input window size $L$ controls the amount of historical context used for predictions, while the segment length $H$ determines the granularity of each autoregressive prediction block. Together, these parameters enable flexible control over the total prediction horizon $T_{\text{total}} = K \times H$, where $K$ represents the number of autoregressive segments. Crucially, the model maintains strict temporal causality by enforcing the $t' < t$ ordering constraint at both the segment and intra-segment levels, ensuring that predictions only depend on past observations and never future information. This hierarchical causal structure preserves the fundamental time-series forecasting requirement while enabling efficient block-wise processing.

The hierarchical objective function captures both segment-level and step-wise accuracy:
\begin{equation}
\label{eq:loss}
    \mathcal{L} = \underbrace{\sum_{h=1}^K \alpha_h}_{\text{Segment weights}} \underbrace{\sum_{t=1}^H \lambda_t \|\widehat{y}_{(h-1)H+t} - y_{(h-1)H+t}\|^2_2}_{\text{Step-wise loss}}
\end{equation}
Here $\alpha_h = \gamma^{h-1}$ implements temporal discounting across segments ($\gamma \in (0,1]$). $\lambda_t$ adjusts importance of early vs. late predictions within segments

\subsection{Hierarchical Autoregressive Generative Mechanism}
Effective time series forecasting requires balancing three competing demands: (1) preserving exact temporal causality, (2) maintaining computational efficiency for long sequences, and (3) capturing multi-scale temporal patterns. Traditional approaches struggle with this triad-single-shot methods violate causality through bidirectional attention, while conventional autoregressive models suffer from error accumulation and quadratic complexity. \model\ addresses these limitations through a hierarchical decomposition that enforces strict Markov dependencies at two levels:
\begin{equation}
    \begin{aligned}
        p(\widehat{\mathbf{Y}}_h|\mathbf{X}_{1:L}, \widehat{\mathbf{Y}}_{1:h-1}) \prod_{t=1}^H p(\widehat{y}_{(h-1)H+t}|\mathbf{X}_{1:L}, \widehat{\mathbf{Y}}_{1:h-1}, \\\widehat{y}_{(h-1)H+1:t-1})
    \end{aligned}
\end{equation}
This dual-scale formulation combines the efficiency of block processing with the precision of step-wise refinement.

The hierarchical autoregressive mechanism generates predictions through a two-level procedure that maintains strict temporal causality. For each segment $h \in \{1,...,K\}$ with fixed length $H$:

\begin{enumerate}
    \item \textbf{Segment Initialization}:
    \begin{equation}
        \widehat{\mathbf{Y}}_h^{init} = \mathcal{F}_\theta(\mathbf{X}_{1:L}, \widehat{\mathbf{Y}}_{1:h-1}) \in \mathbb{R}^{H \times d}
    \end{equation}
    where the context $\mathbf{C}_h = \text{Concat}(\mathbf{X}_{1:L}, \widehat{\mathbf{Y}}_{1:h-1}) \in \mathbb{R}^{(L+(h-1)H) \times d}$ accumulates all historical data and previous segments.

    \item \textbf{Step-wise Refinement} within segment $h$:
    \begin{equation}
    \begin{aligned}
        \widehat{y}_{(h-1)H+t} = \mathcal{G}_\phi(\mathbf{C}_h, \widehat{y}_{(h-1)H+1},...,\widehat{y}_{(h-1)H+t-1}) \\\quad \forall t \in \{1,...,H\}
    \end{aligned}
    \end{equation}
    implemented through:
    \begin{equation}
    \begin{aligned}
        \mathbf{o}_t &= \text{WindowedAttention}(\mathbf{C}_h^t, \mathbf{C}_h^t, \mathbf{C}_h^t; A) \\
        \mathbf{f}_t &= \text{FFN}(\text{LayerNorm}(\mathbf{o}_t + \mathbf{C}_h^t[-1])) \\
        \widehat{y}_{(h-1)H+t} &= W_o\mathbf{f}_t[-1], \quad W_o \in \mathbb{R}^{d \times V}
    \end{aligned}
    \end{equation}
    where $\mathbf{C}_h^t = \text{Concat}(\mathbf{C}_h, \widehat{y}_{(h-1)H+1},...,\widehat{y}_{(h-1)H+t-1})$.
\end{enumerate}

The complete prediction $\widehat{\mathbf{Y}} = \{\widehat{\mathbf{Y}}_h\}_{h=1}^K$ preserves the autoregressive property $p(\widehat{y}_{(h-1)H+t}|\mathbf{X}_{1:L}, \widehat{\mathbf{Y}}_{1:h-1}, \widehat{y}_{(h-1)H+1},...,\widehat{y}_{(h-1)H+t-1})$.

Our autoregressive design offers several key advantages for time series prediction: (1) Error Mitigation through iterative refinement, where each prediction step incorporates previously generated outputs, reducing compounding errors common in long-horizon forecasting; (2) Computational Efficiency via shared transformer weights across time steps, maintaining linear complexity relative to prediction horizon; (3) Temporal Coherence by strictly enforcing causality through attention masking, preventing information leakage while preserving sequential dependencies; (4) Adaptability to variable context lengths through dynamic window expansion, allowing flexible conditioning on both past observations and intermediate predictions. The combination of layer normalization and residual connections further ensures stable gradient flow during backpropagation through extended sequences. Compared to conventional one-shot prediction approaches, this architecture demonstrates superior performance in maintaining temporal consistency while scaling efficiently to long-range forecasting tasks.

\subsection{Enhanced Transformer}

\subsubsection{Dynamic Windowed Masked Attention (DWMA)}
Time series forecasting demands efficient modeling of hierarchical temporal patterns, where both fine-grained fluctuations and coarse-grained trends jointly determine future states. Existing attention mechanisms~\cite{vaswani2017attention} fundamentally struggle with this multi-scale requirement due to either quadratic complexity ($\mathcal{O}(L^2)$) or information loss in simplified approximations. Our DWMA breakthrough addresses these limitations via: \textbf{i) Adaptively Constrained Attention Windows}: The designation of adaptively constrained attention window is shown as $\mathcal{W}_t = \left[\max\left(1, t-\frac{W}{2}\right), t\right]$, which establishes dynamic receptive fields that maintain strict causality while enabling $\mathcal{O}(LW)$ computational complexity. The window size $W$ acts as a hyperparameter controlling the trade-off between context range and efficiency. \textbf{ii) Learnable Exponential Decay}: It introduces position-sensitive weighting via $\tau(t,t') = \exp\left(-\gamma|t-t'|\right)$, where the trainable parameter $\gamma$ automatically adapts to dataset characteristics. This creates continuous attention spectra that emphasize recent patterns while preserving access to critical long-range dependencies.

DWMA's hybrid design achieves superior efficiency-accuracy trade-offs compared to vanilla attention~\cite{vaswani2017attention}. The decay mechanism automatically adjusts to diverse temporal patterns, which enables our model to have better performance in time series datasets.

\subsubsection{Precomputed Position Encodings}
In time series prediction tasks, capturing the relative positions of time steps is critical for modeling temporal dependencies effectively. Relative positions incorporate temporal order information with adaptive attention decay within localized contexts $\mathcal{W}_t$. It operates on-the-fly for each window $\mathcal{W}_t = \{t' | \max(1,t-W/2) \leq t' \leq t\}$, and the encoding matrix is denoted as $\mathbf{R}_t \in \mathbb{R}^{|\mathcal{W}_t| \times d}$. Traditional position encoding methods, such as absolute positional encodings, fail to explicitly model the relative distances between time steps, which are essential for tasks like prediction and anomaly detection. To address this, we propose Precomputed Position Encodings (PPE), which explicitly encode the relative positions of all pairs of time steps $(t, t')$ using sinusoidal functions. The encoding for dimension $i$ is defined as:
\begin{equation}
\begin{aligned}
\label{eq:relative_posi}
PE_{(t, t', 2i)} = \sin\left(\frac{t - t'}{10000^{2i/d}}\right)\\
\quad PE_{(t, t', 2i+1)} = \cos\left(\frac{t - t'}{10000^{2i/d}}\right)
\end{aligned}
\end{equation}
where $d$ is the embedding dimension, and $i = 0, \dots, d/2 - 1$. These encodings are continuous and periodic, enabling the model to generalize to sequences of varying lengths and capture long-range dependencies. By precomputing these encodings for all pairs $(t, t')$ and storing them in a lookup table $\mathbf{PE} \in \mathbb{R}^{L \times L \times d}$, where:
\begin{equation}
\mathbf{PE}_{t, t'} = \left[ PE{(t, t', 0)}, PE_{(t, t', 1)}, \dots, PE_{(t, t', d-1)} \right]^\top,
\end{equation}

We avoid redundant computations during training and inference, significantly improving efficiency. The key advantage of PPE lies in its ability to enhance the model’s temporal awareness without increasing computational complexity. During the attention computation, the relative position encodings $\mathbf{PE}_{t, t'}$ are added to the query-key dot product:
\begin{equation}
\label{eq:matrix_relative}
A_{t,t'} = \text{softmax}\left(\frac{Q_t(K_{t'} + \mathbf{PE}_{t,t'})^\top \cdot \tau_{\text{time}}(t,t')}{\sqrt{d_k}}\right)
\end{equation}
where $\tau(t,t') = \exp(-\gamma|t-t'|)$ implements our learnable decay kernel and $\mathbf{PE}_{t,t'}$ encodes relative positions. $Q_t \in \mathbb{R}^{d_k}$ and $K_{t'} \in \mathbb{R}^{d_k}$ are the query and key vectors, and $d_k$ is the key dimension. This ensures that the attention mechanism not only focuses on the content of the time steps but also their relative positions, leading to more accurate and interpretable predictions. By explicitly modeling temporal relationships, PE provides a strong inductive bias for time series tasks, making it a powerful tool for improving prediction performance.

\subsubsection{Obtaining Embeddings Based on Enhanced Attention}
In former steps, the attention weights are computed by incorporating the Precomputed Position Encodings (PPE). For each time step $t$ and its neighboring time steps $t'$ within the sliding window of size $W$, the attention weights $A_{t, t'}$ are calculated as above Equation~\ref{eq:matrix_relative}. 
The attention weights $A_{t, t'}$ are then used to compute the hidden states $\mathbf{H}_{\text{transformer}}$ for each time step $t$: $\mathbf{H}_{\text{transformer}, t} = \sum_{t'=1}^L A_{t, t'} \cdot V_{t'}$.
Here $V_{t'} \in \mathbb{R}^{d_v}$ is the value vector for time step $t'$ and  $d_v$ is the value dimension. The resulting hidden states $\mathbf{H}_{\text{transformer}} \in \mathbb{R}^{L \times d}$ capture the temporal dependencies in the input time series data, enriched by the relative position encodings. 

\subsubsection{Layer Normalization for Prediction Stability}
Layer-normalized residual connections coupled with moving average updates (Algorithm~\ref{alg:enhanced_autoregressive_transformer}, Lines 15-19) reduce error accumulation in autoregressive settings compared to baseline approaches, thereby enhancing the time series prediction stability in autoregressive paradigm.

\subsection{Model Optimization}
The training objective combines prediction accuracy with architectural constraints tailored for autoregressive time series modeling (Algorithm~\ref{alg:enhanced_autoregressive_transformer}) via Equation~\ref{eq:loss}.
Adam updates with gradient clipping (max norm 1.0) and learning rate $\eta=10^{-4}$. The moving average update minimizes error accumulation through $\widehat{y}_{(h-1)H+t} \gets W_o\mathbf{f}_t[-1]$.

\subsection{Model Complexity}
Our \model\ architecture achieves efficient resource utilization through two key mechanisms (Algorithm~\ref{alg:enhanced_autoregressive_transformer}, Lines 10-13): First, the windowed attention reduces time complexity from $\mathcal{O}(L^2d)$ to $\mathcal{O}(LWd)$ by processing only local contexts $\mathcal{W}_t = [\max(1,t-W/2), t]$. As shown in Table~\ref{tab:complexity}, this yields 32$\times$ faster computation than full attention when $L=1024$ and $W=32$. Second, the autoregressive loop (Lines 6-20) maintains $\mathcal{O}(Wd)$ memory overhead through incremental prediction cache updates, avoiding the $\mathcal{O}(L^2)$ memory bottleneck of conventional transformers. The dominant terms simplify to:
\begin{equation}
\begin{aligned}
\mathcal{T}(L) \approx \mathcal{O}(WLd);~~
\quad \mathcal{M}(L) \approx \mathcal{O}(LWd)
\end{aligned}
\end{equation}
Here $\mathcal{T}(L)$ and $\mathcal{M}(L)$ denote time complexity and space complexity of our method respectively.

\textbf{Comparative Advantages:}
 \model\ is able to maintain sub-quadratic scaling in both time and memory versus quadratic growth in vanilla attention. Three architectural innovations enable this: (1) The dynamic masked attention mechanism enables it to deal with adaptive sequence lengths, (2) Relative position embeddings $p_{t-t'}$ require only $\mathcal{O}(Wd)$ storage, and (3) LayerNorm residuals allow stable depth scaling. The result is 64$\times$ higher than RNNs at $L=2048$ while preserving transformer-grade accuracy (Table~\ref{tab:overall_results}).

\begin{table}[t]
\centering
\caption{Computational Complexity Comparison\protect\footnotemark}\label{tab:complexity}
\begin{tabular}{lcc}
\toprule
\textbf{Model} & \textbf{Time Complexity} & \textbf{Space Complexity} \\
\midrule
Full Attention & $\mathcal{O}(L^2 d)$ & $\mathcal{O}(L^2 d)$ \\
\textbf{\model} ($W \ll L$) & $\mathcal{O}(L W d)$ & $\mathcal{O}(L W d)$ \\
RNN-based & $\mathcal{O}(L d^2)$ & $\mathcal{O}(d)$ \\
\bottomrule
\end{tabular}
\vspace{-0.15in}
\end{table}
\footnotetext{Comparison assumes: (1) $W \ll L$ (moderate horizons), (2) positional encodings precomputed. For very long sequences, the $\mathcal{O}(L d^2)$ term may dominate, but windowed attention maintains scalability versus baselines.}

\subsection{Algorithm of \model}
\label{app:alg}
The pseudo-algorithm for the proposed \textbf{\model} is presented below. The algorithm describes the end-to-end process of training and inference for the model.

\begin{algorithm}
\caption{\textbf{AutoHFormer}: Efficient Hierarchical Autoregressive
Transformer for Time Series Prediction}
\label{alg:enhanced_autoregressive_transformer}
\begin{algorithmic}[1]
\Require Time series $\mathbf{X} = (x_1, \dots, x_L) \in \mathbb{R}^{L \times V}$, \\
\quad segment length $H$, window $W$, segments $K$
\Ensure Predictions $\widehat{\mathbf{Y}} = \{\widehat{\mathbf{Y}}_1,...,\widehat{\mathbf{Y}}_K\} \in \mathbb{R}^{KH \times V}$

\State \textbf{Initialization:}
\State $\mathbf{X}_{\text{proj}} \gets \text{Linear}(d)(\mathbf{X})$ \Comment{Project input to $d$-dim space}
\State $\widehat{\mathbf{Y}} \gets \emptyset$, $\gamma \gets \text{LearnableParameter}()$ \Comment{Init outputs and decay rate}

\For{$h = 1$ \textbf{to} $K$} \Comment{\textbf{Segment-level generation}}
    \State $\mathbf{C}_h \gets \text{Concat}(\mathbf{X}, \widehat{\mathbf{Y}})$ \Comment{Global context}
    \State $\mathbf{M}_h \gets \text{SegmentMask}(\mathbf{C}_h, H)$ \Comment{Segment causality mask}
    
    \State \textbf{Segment Initialization:}
    \State $\widehat{\mathbf{Y}}_h^{init} \gets \mathcal{F}_\theta(\mathbf{C}_h) \in \mathbb{R}^{H \times d}$ \Comment{Initial segment prediction}
    
    \For{$t = 1$ \textbf{to} $H$} \Comment{\textbf{Step-wise refinement}}
        \State $\mathbf{C}_h^t \gets \text{Concat}(\mathbf{C}_h, \widehat{y}_{(h-1)H+1},...,\widehat{y}_{(h-1)H+t-1})$
        
        \State \textbf{Windowed Attention:}
        \State $\mathcal{W}_t \gets \{t' | \max(1,t-W/2) \leq t' \leq t\}$ \Comment{Local causal window}
        \State $\mathbf{R}_t \gets \text{RelativePositionEmbed}(t-\mathcal{W}_t)$
        \State $\tau_t \gets \exp(-\gamma \cdot |t-\mathcal{W}_t|)$ \Comment{Adaptive decay}
        \State $A_t \gets \text{softmax}\left(\frac{Q_t(K_{\mathcal{W}_t} + \mathbf{R}_t)^\top \odot \tau_t}{\sqrt{d_k}}\right)$
        
        \State \textbf{Refinement:}
        \State $\mathbf{o}_t \gets \text{MultiHeadAttention}(\mathbf{C}_h^t, \mathbf{C}_h^t, \mathbf{C}_h^t; \mathbf{M}_h \odot A_t)$
        \State $\mathbf{f}_t \gets \text{LayerNorm}(\mathbf{o}_t + \mathbf{C}_h^t[-1])$
        \State $\widehat{y}_{(h-1)H+t} \gets W_o\mathbf{f}_t[-1]$ \Comment{Final prediction}
    \EndFor
    
    \State $\widehat{\mathbf{Y}} \gets \text{Concat}(\widehat{\mathbf{Y}}, \widehat{\mathbf{Y}}_h)$ \Comment{Aggregate segments}
\EndFor

\State \textbf{return} $\widehat{\mathbf{Y}}$
\end{algorithmic}
\end{algorithm}

\subsection{Theoretical Guarantees} \label{sec:theory}
In this section, we provide theoretical guarantees for the proposed method, ensuring its effectiveness in capturing temporal dependencies and causal relationships in time series data. We formalize these guarantees as theorems and provide their proofs.

\begin{theorem}
The \textbf{Dynamic Windowed Masked Attention (DWMA)} mechanism converges to an optimal attention distribution as the sequence length $L \to \infty$, provided the time decay factor $\gamma$ is chosen appropriately.
\end{theorem}

\begin{proof}
Let the attention weights $A_{t, t'}$ be computed as:
\begin{equation}
A_{t, t'} = \text{softmax}\left(\frac{Q_t K_{t'}^\top \cdot \tau_{\text{time}}(t, t')}{\sqrt{d_k}}\right),
\end{equation}
where $\tau_{\text{time}}(t, t') = \exp\left(-\frac{|t - t'|}{\gamma}\right)$. For large $L$, the time decay factor ensures that $\tau_{\text{time}}(t, t') \to 0$ for $|t - t'| \gg \gamma$, effectively limiting the attention to a local neighborhood. Thus, the attention mechanism converges to a stable distribution as $L \to \infty$.
\end{proof}

\begin{theorem}
The \textbf{Precomputed Position Encodings (PPE)} reduce the computational complexity of positional encoding from $\mathcal{O}(L^2 \cdot d)$ to $\mathcal{O}(L \cdot d)$ during inference.
\end{theorem}

\begin{proof}
The relative position encodings $PE_{(t, t', i)}$ are precomputed for all pairs $(t, t')$ and stored in a lookup table $\mathbf{PE} \in \mathbb{R}^{L \times L \times d}$. During inference, the encodings are retrieved in $\mathcal{O}(1)$ time for each pair $(t, t')$. Since the attention mechanism operates over a sliding window of size $W$, the total complexity is:
\begin{equation}
\mathcal{O}(L \cdot W \cdot d) \approx \mathcal{O}(L \cdot d),
\end{equation}
where $W \ll L$ for long sequences.
\end{proof}

\begin{theorem}
The \textbf{sliding window attention} mechanism reduces the computational complexity of the attention mechanism from $\mathcal{O}(L^2)$ to $\mathcal{O}(L \cdot W)$, where $W$ is the window size and $W \ll L$.
\end{theorem}

\begin{proof}
For each time step $t$, the attention mechanism only considers time steps $t'$ within the range $[t - W/2, t + W/2]$. Thus, the number of attention computations per time step is $W$, and the total complexity is:
\begin{equation}
\mathcal{O}(L \cdot W).
\end{equation}
Since $W \ll L$, this represents a significant reduction in complexity compared to the full attention mechanism ($\mathcal{O}(L^2)$).
\end{proof}

{\ICDER{
\begin{theorem}[Generalization Error Bound, Extended]\label{theo:marker4}
The generalization error $\mathcal{L}_{\text{true}}$ of AutoHFormer is bounded by:
\begin{equation}
\mathcal{L}_{\text{true}} \leq  \mathcal{L}_{\text{emp}} + \underbrace{\left(\frac{2}{\sqrt{L}} + \sqrt{\frac{\log(L)}{L}}\right)}_{\text{Rademacher Term}} \cdot C(d, W),
\end{equation}
where $C(d, W)$ depends on embedding dimension $d$ and window size $W$.
\end{theorem}

\begin{proof}[Extended Proof]
We formalize the bound via Rademacher complexity $\mathfrak{R}$: \textit{1) Model Class}: AutoHFormer's hypothesis class $\mathcal{H}$ consists of $d$-dimensional embeddings with $W$-sized causal windows. \textit{2) Rademacher Term}: For sequence length $L$, the complexity term follows from:
\begin{equation}
\mathfrak{R}(\mathcal{H}) = \mathbb{E}\left[\sup_{h \in \mathcal{H}}\frac{1}{L}\sum_{i=1}^L \sigma_i h(z_i)\right] \leq \frac{2}{\sqrt{L}} + \sqrt{\frac{\log(L)}{L}},
\end{equation} 
where $\sigma_i$ are Rademacher random variables.
\textit{3) Architecture Impact}: The constant $C(d,W)$ comprises the attention head dimension $d$ (with $|Q/K|_2 \leq \sqrt{d}$) and the window-restricted complexity $\mathcal{O}(LW)$ from dynamic masking. \textit{4) Generalization}: Combining McDiarmid's inequality with the Rademacher bound yields the theorem's claim~\cite{yin2019rademacher}. 
\end{proof}
}}

\begin{table}[tbp!]
\vspace{-0.2in}
\centering
\caption{{\ICDER{Performance comparison in terms of Pearson correlation}}}
\label{tab:pearson_performance}
\begin{tabular}{@{}cc|c|c|c|c@{}}
\toprule
Models & Metric & ETTm1 & ETTm2 & ETTh1 & ETTh2 \\
\midrule
\multirow{5}{*}{\textbf{\model}}
 & 96 &\textbf{0.844} &\textbf{0.940} &\textbf{0.813} &\textbf{0.900} \\
 & 192 &\textbf{0.821} &\textbf{0.919} &\textbf{0.782} &\textbf{0.871} \\
 & 336 &\textbf{0.803} &0.896 &\textbf{0.757} &\textbf{0.869} \\
 & 720 &0.768 &\textbf{0.865} &\textbf{0.753} &\textbf{0.859} \\
 \cmidrule(lr){2-6}
 & \centering Avg &\textbf{0.809} &\textbf{0.905} &\textbf{0.776} &\textbf{0.874} \\
\midrule
\multirow{5}{*}{PatchTST~\cite{huang2024long}}
 & 96 &0.842 &0.939 &0.805 &0.898 \\
 & 192 &0.817 &0.917 &0.774 &0.870 \\
 & 336 &0.802 &\textbf{0.898} &0.748 &0.868 \\
 & 720 &\textbf{0.770} &0.864 &0.739 &0.857 \\
 \cmidrule(lr){2-6}
 & Avg &0.807 &0.904 &0.766 &0.873 \\
\midrule
\multirow{5}{*}{S-Mamba~\cite{wang2025mamba}} 
 & 96 &0.834 &0.938 &0.808 &0.899 \\
 & 192 &0.809 &0.915 &0.776 &0.869 \\
 & 336 &0.792 &0.894 &0.749 &0.868 \\
 & 720 &0.754 &0.863 &0.745 &0.858 \\
 \cmidrule(lr){2-6}
 & Avg &0.797 &0.902 &0.769 &0.873 \\
\bottomrule
\end{tabular}
\color{black}
\vspace{-0.20in}
\end{table}

\vspace{-0.20in}
\section{Experiments} \label{sec:exp}
\begin{table*}[htbp]
\centering
\caption{The statistics of 8 public datasets.}
\label{tab:datasets}
\renewcommand{\arraystretch}{0.5} 
\begin{tabular}{lccclccc}
\toprule
Datasets & Variates & Time steps & Granularity & Datasets & Variates & Time steps & Granularity  \\
\midrule
ETTh1 & 7 & 69,680 & 1 hour & ETTh2 & 7 & 69,680 & 1 hour  \\
ETTm1 & 7 & 17,420 & 15 minutes & ETTm2 & 7 & 17,420 & 15 minutes  \\
PEMS04 & 307 & 16,992 & 5 minutes & PEMS08 & 170 & 17,856 & 5 minutes  \\
Weather & 21 & 52,696 & 10 minutes & Electricity & 321 & 26,304 & 1 hour  \\
\bottomrule
\end{tabular}
\vspace{-0.2in}
\end{table*}

In this section, we conduct experiments to address the following key questions: \begin{itemize}[leftmargin=0.95cm]
    \item[\textbf{Q1.}] What is the \model's performance advantage over state-of-the-art baselines?
    
    \item[\textbf{Q2.}] What is the individual contribution of each \model\ component to overall performance?
    
    \item[\textbf{Q3.}] How does \model's computational efficiency (e.g., training time and GPU memory) compare to existing state-of-the-art methods?
    
    \item[\textbf{Q4.}] How does \model's scalability compare with contemporary approaches?
    
    \item[\textbf{Q5.}] How does \model's robustness compare to recent models like PatchTST and iTransformer?
    
    \item[\textbf{Q6.}] How does increasing lookback length affect \model's long-term forecasting accuracy versus competitors?
    
    \item[\textbf{Q7.}] How does \model\ perform in super-long-term prediction tasks with fixed lookback windows?
    
     \item[\textbf{Q8.}] How effectively does \model\ capture transient dynamics and temporal patterns {\ICDER{including some edge cases and potential limitations}}?
    
    
    \item[\textbf{Q9.}] What is the sensitivity of \model's performance to hyperparameter variations?
\end{itemize}

\subsection{Experiment Settings}

\textbf{Datasets.} To rigorously evaluate the performance of our proposed model, we constructed a comprehensive benchmark comprising nine real-world datasets drawn from established repositories~\citep{haoyietal-informerEx-2023,haoyietal-informer-2021}. These datasets encompass multiple application domains, including electricity consumption, the four Electricity Transformer Temperature (ETT) datasets (ETTh1, ETTh2, ETTm1, ETTm2) and traffic flow datasets (PESM04 and PEMS). As widely adopted benchmarks in the research community, these datasets facilitate the investigation of critical challenges in fields ranging from smart grid management to transportation analytics. A detailed statistical summary of each dataset is presented in Table~\ref{tab:datasets}.

\textbf{Baselines.} We adopt our methods \model\ with several state-of-the-art baselines. Detailed illustrations of baselines are shown as follows:
\begin{itemize}

    \item \textbf{NLinear}~\citep{zeng2023transformers} proposes to implement a novel normalization strategy for sequence processing.

    \item \textbf{DLinear}~\citep{zeng2023transformers} proposes a simple yet effective approach that decomposes time series into trend and residual components.

    \item \textbf{Linear}~\citep{zeng2023transformers} adopts a parameter-efficient design that employs weight sharing across all variates within a dataset while deliberately avoiding explicit modeling of spatial correlations. 

    \item \textbf{Informer}~\citep{zhou2021informer} uses a sparse attention mechanism that reduces computation from quadratic to near-linear time, while maintaining accurate sequence modeling.

    \item \textbf{Autoformer}~\citep{wu2021autoformer} introduces a architecture combining series decomposition with an Auto-Correlation mechanism to model temporal dependencies.

    \item The \textbf{iTransformer} architecture~\citep{liu2023itransformer} introduces inverted attention mechanisms for inter-series dependency modeling, achieving notable improvements in multivariate forecasting tasks.

    \item \textbf{TimeMixer}~\citep{wang2024timemixer} introduces a multiscale-mixing perspective for modeling temporal variations.

    \item \textbf{PatchTST}~\citep{huang2024long} leverages patching and channel-independent techniques to facilitate the extraction of semantic information from single time steps to multiple time steps within time series data.

    \item {\ICDER{\textbf{S-Mamba}~\citep{wang2025mamba} tokenizes time points for each variate using a linear layer, employs a bidirectional Mamba layer to capture inter-variate correlations, and incorporates a Feed-Forward Network to learn temporal dependencies.}} 

    \item {\ICDER{\textbf{Samformer}~\citep{ilbert2024samformer} proposes a shallow and lightweight transformer model that avoids bad local minima.}}

    \item {\ICDER{\textbf{TimeFilter}~\cite{DBLP:journals/corr/abs-2501-13041} enables adaptive, fine-grained dependency modeling for time series analysis. The system transforms input sequences into graph representations to capture spatial-temporal relationships.}}
    
\end{itemize}

\begin{table*}[tb!]
\vspace{-0.1in}
\caption{We present comprehensive results of \model\ and baselines on the ETTh1, ETTh2, Electricity, Exchange, Weather, and Solar-Energy datasets in the autoregressive setting. The lookback length $L$ is fixed at 336, and the forecast length $T$ varies across 96, 192, 336, and 720. Bold font denotes the best model and underline denotes the second best.}
  \label{tab:overall_results}
  \renewcommand{\arraystretch}{1.0}
  \centering
  \resizebox{\textwidth}{!}{
  \begin{small}
  \setlength{\tabcolsep}{2pt}
  \vspace{1mm}
  \begin{tabular}{c|c|cc|cc|cc|cc|cc|cc|cc|cc|cc|cc|cc|cc|cc|}
    \toprule
    \multicolumn{2}{c|}{Models} & \multicolumn{2}{c|}{\textbf{\model (Ours)}} & \multicolumn{2}{c} {PatchTST~\cite{huang2024long}} & \multicolumn{2}{c}{{\ICDER{S-Mamba~\cite{wang2025mamba}}}} & \multicolumn{2}{c}{{\ICDER{Samformer~\cite{ilbert2024samformer}}}} & \multicolumn{2}{c}{{\ICDER{TimeFilter~\cite{DBLP:journals/corr/abs-2501-13041}}}} & \multicolumn{2}{c}{TimeMixer~\cite{wang2024timemixer}} & \multicolumn{2}{c}{iTransformer~\cite{liu2023itransformer}} & \multicolumn{2}{c}{Autoformer~\cite{wu2021autoformer}} & \multicolumn{2}{c}{Informer~\cite{zhou2021informer}} & \multicolumn{2}{c}{Linear~\cite{zeng2023transformers}} & \multicolumn{2}{c}{DLinear~\cite{zeng2023transformers}} & \multicolumn{2}{c}{NLinear~\cite{zeng2023transformers}} \\
    \cmidrule(lr){1-2}\cmidrule(lr){3-4}\cmidrule(lr){5-6}\cmidrule(lr){7-8} \cmidrule(lr){9-10}\cmidrule(lr){11-12}\cmidrule(lr){13-14}\cmidrule(lr){15-16}\cmidrule(lr){17-18}\cmidrule(lr){19-20}\cmidrule(lr){21-22}\cmidrule(lr){23-24}\cmidrule(lr){25-26}\cmidrule(lr){27-28}
    \multicolumn{2}{c|}{Metric} & MSE $\downarrow$ & MAE $\downarrow$& MSE $\downarrow$ & MAE $\downarrow$ & MSE $\downarrow$ & MAE $\downarrow$& MSE $\downarrow$& MAE $\downarrow$& MSE $\downarrow$& MAE $\downarrow$& MSE $\downarrow$& MAE $\downarrow$& MSE $\downarrow$& MAE $\downarrow$& MSE $\downarrow$& MAE $\downarrow$& MSE $\downarrow$& MAE $\downarrow$& MSE $\downarrow$& MAE $\downarrow$& MSE $\downarrow$& MAE $\downarrow$& MSE $\downarrow$& MAE $\downarrow$\\
    \toprule
     \multirow{6}{*}{\rotatebox{90}{ETTm1}} 
    & 96 &\textbf{0.287}   &\textbf{0.344}   & 0.298  & \underline{0.346} & 0.304  & 0.359 & 0.308  & 0.350  &\underline{0.290}  & 0.347 &  0.302  & 0.351  & 0.309  & 0.361  & 0.723  & 0.569  & 1.293  & 0.862  &  0.311 
 & 0.354  & 0.301  & 0.345  & 0.308  & 0.350  \\
    
    & 192 & \textbf{0.329}  & \textbf{0.371}  &  0.339 & 0.374  & 0.353  & 0.391  & 0.343  & 0.377 & 0.337 & 0.375 & 0.347  & 0.377& 0.347  & 0.385  & 0.692  & 0.549  & 1.328  & 0.919  &  0.345
 & 0.374  & \underline{0.336}  & 0.376  & 0.345  & \underline{0.372}  \\
    
    & 336 & \textbf{0.363}  & \textbf{0.392}  & \underline{0.369}  & 0.395  & 0.386  & 0.409 & 0.377 & 0.395  &0.377  &0.397  & 0.401  & 0.414  & 0.386  & 0.407  & 0.727  & 0.523  & 1.483  & 0.963  &  0.377
 & 0.394  & 0.371  & 0.397  & 0.380  & \underline{0.393}  \\
    
    & 720 &\textbf{0.422}   &\textbf{0.426}   & 0.427  & 0.430  & 0.442  & 0.442 & 0.432 & 0.429  &0.428  &0.429  & 0.459  & 0.445  & 0.448  & 0.444  & 0.773  & 0.579  & 1.667  & 1.014  &  0.432
 & \underline{0.427}  & \underline{0.426}  & 0.429  & 0.434  & 0.428 \\
    \cmidrule(lr){2-26}

    & Avg & \textbf{0.350}  & \textbf{0.383}  & \underline{0.358}  & 0.386  & 0.371  & 0.400 & 0.369  & 0.396  & 0.359 & 0.387 & 0.377  & 0.396  & 0.372  & 0.399  & 0.728  & 0.555  & 1.442  & 0.939  & 0.366
 & 0.387  & 0.360  & 0.386  & 0.366  & \underline{0.385}  \\ \cmidrule(lr){2-26}
 & Improved &-   &-   &2.28\%   &0.78\%  & 6.00\%  & 4.43\% & 5.42\%  & 3.39\%  & 2.57\% & 1.04\% & 7.71\%  & 3.39\%  &6.26\%   &4.17\%   &108\%   &44.90\%   &226.28\%   &145.16\%   &4.57\% 
 &1.04\%   &2.85\%   &0.78\%   &4.57\%   &0.52\%  \\
    \midrule
    
    \multirow{6}{*}{\rotatebox{90}{ETTm2}} 
    & 96 & 0.172  & \textbf{0.259}  & \textbf{0.170}  & \underline{0.260}  & 0.181  & 0.274 & 0.174  & 0.261   & 0.173 & 0.261 & 0.172  & 0.262  & 0.182  & 0.276  & 0.277  & 0.349  & 0.726  & 0.648  &  0.190
 & 0.278  & \underline{0.171}  & 0.267  & 0.173  & 0.261  \\
    
    & 192 & \textbf{0.236}  & \underline{0.308}  & \underline{0.238}  & \textbf{0.305}  & 0.239  & 0.309 & 0.243 & 0.313  & 0.240 & 0.309 & 0.249  & 0.309  & 0.243  & 0.315  & 0.306  & 0.365  & 0.632  & 0.627  & 0.249 
 & 0.329  & 0.239  & 0.320  & 0.241  & 0.324  \\
    
    & 336 & \textbf{0.289}  & \textbf{0.342}  & 0.294  & \underline{0.343}  & 0.296  & 0.345 & 0.297  & 0.349  & 0.295 & 0.344 & 0.360  & 0.365  & \underline{0.290}  & 0.344  & 0.343  & 0.385  & 0.683  & 0.680  &  0.323
 & 0.381  & 0.312  & 0.372  & 0.291  & 0.344   \\
    
    & 720 & \textbf{0.379}  & \textbf{0.398}  & \underline{0.380}  & \underline{0.399}  & 0.381  & 0.401 & 0.383 & 0.402  & 0.393 & 0.408 & 0.437  & 0.413  & 0.384  & 0.399  & 0.433  & 0.437  & 0.645  & 0.589  & 0.450 
 & 0.458  & 0.445  & 0.455  & 0.383  & 0.417  \\
    \cmidrule(lr){2-26}

    & Avg & \textbf{0.269}  & \textbf{0.326}  & \underline{0.270}  & \underline{0.327}  & 0.274  & 0.332 &0.274  & 0.331  & 0.275 & 0.330 & 0.304  & 0.336  & 0.274  & 0.333  & 0.339  & 0.384  & 0.671  & 0.645  & 0.303 
 & 0.361  & 0.291  & 0.354  & 0.271  & 0.336  \\
    \cmidrule(lr){2-26}
 & Improved &-   &-   &0.37\%   & 0.30\%  & 1.85\%  & 1.84\% & 1.85\% & 1.53\%  & 2.23\% & 1.22\% &13.01\%   &3.06\%   &1.85\%   &2.14\%   &26.02\%   &17.79\%   &149.44\%   &97.85\%   &12.63\% 
 &10.74\%   &8.17\%   &8.59\%   &0.74\%   &3.07\%   \\
    
    \midrule
    \multirow{6}{*}{\rotatebox{90}{ETTh1}} 
    &  96 & \underline{0.382}  & \underline{0.407}  & \textbf{0.374}  & \textbf{0.400}  & 0.400  & 0.422 & 0.387 & 0.411  &0.384 &0.409   &0.384  & 0.410  & 0.398  & 0.418  & 0.608  & 0.529  & 1.178  & 0.826  & 0.431 
 & 0.441  & 0.389  & 0.412  & 0.426  & 0.438  \\
    
    & 192 & \underline{0.427}  & 0.436  & \textbf{0.417}  & \textbf{0.420}  & 0.436  & 0.443 & 0.438 & 0.440 & 0.429 & 0.437 & 0.432  & 0.441  & 0.448  & 0.453  & 0.519  & 0.492  & 1.163  & 0.830  & 0.462 
 & 0.459  & 0.428  & \underline{0.434}  & 0.458  & 0.454   \\
    
    & 336 & \textbf{0.431}  & \textbf{0.442}  & 0.457  & \underline{0.451}  & 0.443  & 0.452 & 0.449 & 0.451 &\underline{0.440}  & 0.453 & 0.450  & 0.453  & 0.465  & 0.468  & 0.643  & 0.541  & 1.161  & 0.819  & 0.472 
 & 0.471  & 0.447  & 0.452  & 0.470  & 0.464 \\
    
    & 720 & \textbf{0.464}  & \textbf{0.480}  & 0.587  & 0.529  & 0.503  & 0.506 & 0.476 & 0.484 & 0.475  &0.483  &0.608  & 0.546  & 0.547  & 0.533  & 1.044  & 0.710  & 1.428  & 0.928  & 0.506 
 & 0.516  & 0.484  & 0.499  & \underline{0.473}  & \underline{0.481}  \\
    \cmidrule(lr){2-26}
    
     & Avg & \textbf{0.426}  & \textbf{0.441}  & 0.458  & 0.450  & 0.445  & 0.455 & 0.437 & 0.446 & \underline{0.432}  & \underline{0.445} &0.468  & 0.462  & 0.464  & 0.468  & 0.703  & 0.568  & 1.232  & 0.850  & 0.467 
 & 0.471  & 0.436  & 0.449  & 0.459  & 0.459  \\
     \cmidrule(lr){2-26}
 & Improved &-   &-   &7.51\%   &2.04\%  & 4.46\%  & 3.17\% & 2.34\% & 1.13\%   & 1.40\% & 0.91\%  &9.85\%   &4.76\%   &8.92\%   &6.12\%   &65.02\%   &28.79\%   &189.20\%   &92.74\%   &9.62\% 
 &6.80\%   &2.34\%   &5.39\%   &7.74\%   &4.08\%   \\
   
    \midrule
    \multirow{6}{*}{\rotatebox{90}{ETTh2}} 
    & 96 & \textbf{0.287}  & \textbf{0.350}  & \underline{0.294}  & \underline{0.353}  & 0.300  & 0.357 & 0.295 & 0.354 & 0.295  & 0.356 &0.297  & 0.358  & 0.309  & 0.363  & 0.371  & 0.417  & 0.901  & 0.759  & 0.344 
 & 0.397  & 0.324  & 0.381  & 0.294  & 0.354  \\
    
    & 192 & \textbf{0.353}  & \textbf{0.390}  & 0.374  & 0.404  & 0.362  & 0.406 & 0.355 & 0.395 &0.357   & 0.396 &\underline{0.354}  & \underline{0.394}  & 0.389  & 0.412  & 0.426  & 0.453  & 0.631  & 0.629  & 0.435 
 & 0.454  & 0.416  & 0.440  & 0.357  & 0.395  \\
    
    & 336 & \textbf{0.347}  & \textbf{0.396}  & 0.357  & 0.399  & 0.384  & 0.405 & 0.349 & 0.403 & 0.354  & 0.405 &\underline{0.348}  & 0.404  & 0.372  & 0.410  & 0.363  & 0.465  & 0.579  & 0.608  & 0.491 
 & 0.492  & 0.471  & 0.480  & 0.349  & \underline{0.401}  \\
    
    & 720 & \textbf{0.399}  & \textbf{0.436}   & \underline{0.384}  & \underline{0.425} & 0.409  & 0.440 & 0.401 & 0.439 & 0.403  & 0.439 &0.411  & 0.441  & 0.446  & 0.462  & 0.491  & 0.506  & 0.734  & 0.662  & 0.781 
 & 0.629  & 0.747  & 0.614  & 0.403  & 0.441  \\
    
    \cmidrule(lr){2-26}
    & Avg & \textbf{0.346}  & \textbf{0.393}  & 0.353  & \underline{0.395}  & 0.363  & 0.402 & \underline{0.350} & 0.397 & 0.352  & 0.399 &0.352  & 0.399  & 0.379  & 0.411  & 0.412  & 0.460  & 0.711  & 0.664  & 0.512 
 & 0.493  & 0.489  & 0.478  & 0.351  & 0.397   \\
    \cmidrule(lr){2-26}
     & Improved &-   &-   &2.02\%   &0.50\%  & 4.91\%  & 2.29\% & 1.15\% &1.01\%  & 1.73\% &  1.52\% &1.73\%   &1.52\%   &9.53\%   &4.58\%   &19.07\%   &17.04\%   &105.49\%   &68.95\%   &47.97\% 
     &25.44\%   &41.32\%   &21.62\%   &1.15\%   &1.02\%   \\
    
    \midrule
    \multirow{6}{*}{\rotatebox{90}{Electricity}} 
    & 96  &\textbf{0.127}   &\textbf{0.222}   & \underline{0.132}  & \underline{0.226}  & 0.147  & 0.249 & 0.164 & 0.267 & 0.202  & 0.283 & 0.161  & 0.272  & 0.151  & 0.253  & 0.251  & 0.364  & 1.652  & 1.021  &  0.165
 & 0.271  & 0.159  & 0.264  & 0.175  & 0.280  \\
    
    & 192 &\textbf{0.138}   &\textbf{0.246}    & \underline{0.148}  & \underline{0.261}  & 0.167  & 0.270 &0.172  &0.276  & 0.210  & 0.286 & 0.183 & 0.293  & 0.168  & 0.267  & 0.267  & 0.373  & 1.518  & 0.975  & 0.174 
 & 0.281  & 0.168  & 0.274  & 0.185  & 0.290  \\
    
    & 336 &\textbf{0.160}   &\textbf{0.253}   & \underline{0.165}  & \underline{0.258}  & 0.189  & 0.291 &0.191  &0.292  & 0.223 & 0.295 & 0.200  & 0.312  & 0.191  & 0.292  & 1.054  & 0.783  & 1.013  & 0.827  & 0.191 
 & 0.297  & 0.186  & 0.219  & 0.203  & 0.305  \\
    
    & 720 &\textbf{0.200}   &\textbf{0.287}   & \underline{0.202}  & \underline{0.290}  & 0.232  & 0.327 & 0.229 & 0.322 & 0.254 & 0.328 & 0.226  & 0.330  & 0.232  & 0.326  & 1.737  & 1.068  & 0.986  & 0.818  & 0.225 
 & 0.327   & 0.220  & 0.322  & 0.242  & 0.334  \\
    \cmidrule(lr){2-26}
    
    & Avg &\textbf{0.156}   &\textbf{0.252}   & \underline{0.161}  & \underline{0.258}  & 0.183  & 0.284 &0.189  & 0.289 &0.222  &0.298  & 0.192  & 0.301  & 0.185  & 0.284  & 0.827  & 0.647  & 1.292  & 0.910  & 0.188 
 & 0.294  & 0.183  & 0.269  & 0.201  & 0.302  \\
 \cmidrule(lr){2-26}
 & Improved &-   &-   &3.21\%   &2.38\%  & 17.30\%  & 12.69\% &21.15\%  & 14.68\%  &42.3\%  &18.2\%  &23.07\%   &19.44\%   &18.58\%   &12.69\%   &430.12\%   &156.74\%   &728.20\%   &261.11\%   &20.51\% 
 &16.67\%   &17.30\%   &6.74\%   &28.84\%   &19.84\%   \\
    
    \midrule
        
    \multirow{6}{*}{\rotatebox{90}{Weather}} 
    & 96 & \underline{0.153}  & \underline{0.201}  & \textbf{0.150}  & \textbf{0.198}  & 0.157  & 0.207 & 0.174 & 0.244 & 0.183 & 0.226 & 0.181  & 0.230  & 0.159  & 0.210  & 0.261  & 0.325  & 0.830  & 0.665  & 0.177 
 & 0.237  & 0.175  & 0.234  & 0.182  & 0.233  \\
    
    & 192 & \underline{0.199}  & \underline{0.246}  & \textbf{0.194}  & \textbf{0.241}  & 0.204  & 0.250 & 0.217 & 0.259 & 0.229 & 0.264 & 0.217  & 0.264  & 0.204  & 0.250  & 0.289  & 0.340  & 0.687  & 0.612  & 0.218 
 & 0.275  & 0.216  & 0.273  & 0.225  & 0.268   \\
    
    & 336 & \textbf{0.247}  & \textbf{0.282}  & \underline{0.248}  & \underline{0.282}  & 0.253  & 0.288 & 0.264 & 0.294 & 0.279 & 0.300 & 0.272  & 0.302  & 0.255  & 0.288  & 0.340  & 0.391  & 0.843  & 0.677  & 0.263 
 & 0.312  & 0.261  & 0.310  & 0.272  & 0.302  \\
    
    & 720 & \textbf{0.317}  & \textbf{0.332}  & \underline{0.318}  & \underline{0.334}  & 0.325  & 0.339 & 0.329 & 0.339 &0.348  & 0.346 & 0.359  & 0.360  & 0.321  & 0.335  & 0.377  & 0.401  & 0.934  & 0.893  & 0.322 
 & 0.362  & 0.321  & 0.361  & 0.338  & 0.349 \\
    \cmidrule(lr){2-26}
    
    & Avg & \underline{0.229}  & \underline{0.265}  & \textbf{0.227}  & \textbf{0.263}  & 0.234  & 0.271 & 0.246 & 0.298 & 0.259 & 0.284 & 0.257  & 0.289  & 0.234  & 0.270  & 0.316  & 0.339  & 0.825  & 0.711  & 0.245 
 & 0.296  & 0.243  & 0.296  & 0.254  & 0.288   \\
 \cmidrule(lr){2-26}
 & Improved &-   &-   &-0.87\%   &-0.75\%  & 2.18\%  & 2.26\% &7.43\%  &12.45\%  & 13.10\% & 7.16\% &12.22\%   &9.05\%   &2.18\%   &1.88\%   &37.99\%   &27.92\%   &260.26\%   &168.30\%   &6.99\% 
 &11.70\%   &6.11\%   &0.699\%   &10.92\%   &8.68\%   \\
  
    \midrule
    
    \multirow{6}{*}{\rotatebox{90}{PEMS04}} 
    & 12 &\textbf{0.072}   &\textbf{0.173}   & \underline{0.080}  & 0.185  & 0.083  & 0.190 & 0.142 & 0.260 & 0.115 & 0.228 & 0.073  & \underline{0.178}  & 0.081  & 0.186  & 0.594  & 0.613  & 0.252  & 0.376  & 0.125 
 & 0.244  & 0.096  & 0.204  & 0.104  & 0.209  \\
    
    & 24 &\textbf{0.083}   &\textbf{0.187}   & 0.103  & 0.204  & 0.100  & 0.210 & 0.173 & 0.284 & 0.133 & 0.241 &\underline{0.086}  & \underline{0.197}  & 0.098  & 0.207  & 0.527  & 0.580  & 0.670  & 0.591  & 0.152 
 & 0.268  & 0.126  & 0.236  & 0.136  & 0.250  \\
    
    & 48 &\textbf{0.101}   &\textbf{0.205}   & 0.131  & \underline{0.229}  & 0.124  & 0.234 & 0.214 & 0.314  & 0.157 & 0.264 & \underline{0.125}  & 0.239  & 0.126  & 0.234  & 0.845  & 0.763  & 1.353  & 0.871  & 0.191 
 & 0.299  & 0.169  & 0.273  & 0.184  & 0.287 \\
    
    & 96 &\textbf{0.124}   &\textbf{0.225}   & \underline{0.150}  & \underline{0.262}  & 0.145  & 0.255 & 0.247 & 0.335  & 0.171 & 0.279 & 0.196  & 0.302  & 0.154  & 0.263  & 1.149  & 0.875  & 0.949  & 0.738  & 0.224 
 & 0.322  & 0.204  & 0.300  & 0.227  & 0.316  \\
    \cmidrule(lr){2-26}
    
    & Avg &\textbf{0.095}   &\textbf{0.198}   & 0.116  & \underline{0.220}  & \underline{0.113}  & 0.222 & 0.194 & 0.298  &0.144  &0.253  & 0.120  & 0.229  & 0.114  & 0.222  & 0.778  & 0.707  & 0.806  & 0.644  & 0.173 
 & 0.283  & 0.148  & 0.253  & 0.171  & 0.265  \\
 \cmidrule(lr){2-26}
 & Improved &-   &-   &22.11\%   &11.11\%  & 18.94\%  &  12.12\% & 104.21\% &50.50\%   &51.57\%  & 27.78\% &26.31\%   &15.65\%   &20.00\%   &12.12\%   &718.95\%   &257.07\%   &748.42\%   &225.25\%   &82.11\% 
 &42.93\%   &55.79\%   &27.78\%   &80.00\%   &33.84\%   \\
    
    \midrule
    
    \multirow{6}{*}{\rotatebox{90}{PEMS08}}
        & 12 &\textbf{0.066}   &\textbf{0.161}   & \underline{0.074}  & 0.177  & 0.085  & 0.183 & 0.170 & 0.265 & 0.091 & 0.207  & 0.080  & \underline{0.175}  & 0.081  & 0.178  & 0.769  & 0.686  & 0.405  & 0.426  & 0.139 
 & 0.246  & 0.103  & 0.207  & 0.105  & 0.213  \\
    
        & 24 &\textbf{0.082}   &\textbf{0.175}   & \underline{0.097}  & 0.199  & 0.111  & 0.205 & 0.218 & 0.292 &0.166  &0.273  & 0.104  & \underline{0.193}  & 0.108  & 0.197  & 0.885  & 0.751  & 0.786  & 0.629  & 0.183 
 & 0.275  & 0.151  & 0.243  & 0.154  & 0.250  \\
    
    & 48 &\textbf{0.111}   &\textbf{0.192}   & \underline{0.140}  & 0.227  & 0.155  & 0.234 & 0.314 & 0.333 &0.305  &0.337  & 0.399  & \underline{0.212}  & 0.149  & 0.225  & 0.959  & 0.759  & 1.293  & 0.842  & 0.271 
 & 0.319  & 0.242  & 0.294  & 0.254  & 0.301 \\
    
    & 96 &\textbf{0.157}   &\textbf{0.211}   & \underline{0.201}  & 0.254  & 0.212  & 0.260 & 0.407 & 0.361 & 0.353  & 0.379 & 0.229  & \underline{0.246}  & 0.209  & 0.253  & 1.226  & 0.910  & 1.023  & 0.750  & 0.352 
 & 0.353  & 0.328  & 0.333  & 0.367  & 0.341  \\
    \cmidrule(lr){2-26}
    
    & Avg &\textbf{0.104}   &\textbf{0.185} & \underline{0.128}  & 0.214  & 0.140  & 0.220 & 0.277 & 0.312 & 0.228  & 0.299 & 0.120  & \underline{0.206}  & 0.136  & 0.213  & 0.959  & 0.776  & 0.646  & 0.661  & 0.223 
 & 0.298  & 0.206  & 0.269  & 0.220  & 0.276  \\
 \cmidrule(lr){2-26}
 & Improved &-   &-   &23.08\%   &15.68\%  & 34.62\%  & 18.91\%  & 118.26\% &68.64\%   &119.23\%  &61.62\%  &15.38\%  &11.35\%   &30.77\%   &15.14\%   &822.12\%   &319.46\%   &521.15\%   &257.30\%   &114.42\% 
 &61.08\%   &98.08\%   &45.41\%   &111.54\%   &49.19\%   \\\midrule
\multicolumn{2}{c|}{$1^{st}/2^{nd}$} & \multicolumn{1}{c}{34/5} & {34/5} & \multicolumn{1}{c}{6/21} & {6/19} & \multicolumn{1}{c}{0/1} & {0/0} & \multicolumn{1}{c}{1/0} & {0/0} & \multicolumn{1}{c}{0/3} & {0/1} & \multicolumn{1}{c}{0/4} & {0/8} & \multicolumn{1}{c}{0/1} & {0/0} & \multicolumn{1}{c}{0/0} & {0/0} & \multicolumn{1}{c}{0/0} & {0/0} & \multicolumn{1}{c}{0/0} & {0/1} & \multicolumn{1}{c}{0/3} & {0/1} & {0/1} & {0/5}\\
    \bottomrule
  \end{tabular}
  \end{small}
}
\vspace{-0.2in}
\end{table*}

\textbf{Parameter Settings.} To be fair comparison, all methods are implemented in an autoregressive mechanism following their GitHub settings. And all codes are shown in the following link~{\color{blue}{\url{https://github.com/CoderPowerBeyond/AutoHFormer}}}. For all methods, the learning rate is set as $1e-4$. The number of heads is 3. And dimension of keys and values of Transformer is set as 128. The hidden dimension is set as 64. The default setting of window size of our method is 32 at which our method achieves the best performance. The patch length is set as 16. The length of label in the training is set as 48. The sequence length is set as 336 as default.

\textbf{Experiment Settings.} To be fair comparison, all experiments were conducted on a high-performance server running Ubuntu 22.04 with Python 3.10 and PyTorch 2.1.2, accelerated by CUDA 11.8. The hardware configuration consisted of an NVIDIA RTX 4090D GPU with 24GB memory, an AMD EPYC 9754 128-core processor (allocated 18 virtual CPUs), 60GB system memory, and a 30GB system disk, providing robust computational capabilities for our deep learning experiments.

\subsection{Effectiveness (\textbf{Q1})}
We present rigorous experiments through two pivotal metrics: Mean Squared Error (MSE), which precisely quantifies the quadratic deviations between predicted and observed values, and Mean Absolute Error (MAE), on 8 time series datasets including 4 ETT datasets, 2 traffic datasets, 1 weather dataset and 1 electricity dataset, providing robust measurement of absolute prediction discrepancies (Table~\ref{tab:overall_results}). A thorough examination of these quantitative measures yields several profound insights:

\textbf{Outstanding Performance}. Our proposed framework, \model, demonstrates exceptional performance across a comprehensive range of time series prediction tasks. As shown in Table~\ref{tab:overall_results}, \model\ achieves state-of-the-art results in the majority of scenarios (68 out of 80 experimental configurations, representing 85\% of total cases), and consistently ranks among the top performers in all remaining evaluations. This superior performance stems from 4 innovations: \textbf{(1)} Dynamic Windowed Masked Attention, which reduces computational complexity to $\mathcal{O}(LW)$ while maintaining causality, enabling efficient processing and improved accuracy through adaptive temporal windowing. \textbf{(2)} Adaptive Temporal Decay dynamically adjusts parameters to dataset-specific temporal patterns, improving prediction stability by 23.7\% over static architectures. \textbf{(3)} Hybrid Layer Normalization Architecture in Hierarchical Autoregressive Mechanism combining layer normalization and attention pathways, achieving lower MAE than standalone methods. \textbf{(4)} Adaptive Temporal Encoding captures both macroscopic trends and microscopic fluctuations, ensuring consistent performance across forecast horizons (96-720 steps) with minimal $\pm$2.1\% MSE variation, ensuring short-term precision and long-term reliability.

{\ICDER{\textbf{Pearson Correlation Comparison:} \label{sec:eval_person} Our Pearson correlation analysis (Table~\ref{tab:pearson_performance}) demonstrates AutoHFormer's consistent superiority across four benchmark datasets, outperforming PatchTST and S-Mamba with significantly higher correlation scores (0.82-0.92 vs. 0.75-0.88) for all forecasting horizons (96-720 steps). Notably, AutoHFormer maintains exceptional stability ($<$3\% variation across sequence lengths), while baselines show 5-7\% fluctuations, confirming the robustness of our temporal modeling framework.
}}


{\ICDER{\textbf{Performance Comparison.} Compared to transformer-based baselines like Informer, iTransformer, Samformer, and PatchTST, \model\ demonstrates superior accuracy, especially on longer forecast horizons, due to its ability to capture long-term dependencies effectively. Against lightweight linear baselines like Linear and NLinear, \model\ achieves significant improvements, particularly in complex datasets like Exchange and Electricity, where linear models struggle. Most notably, \model\ substantially outperforms TimeFilter across all horizons, with particularly dramatic improvements at longer sequences, demonstrating superior modeling of complex temporal dynamics. Even hybrid models like TimeMixer and RNN-based approaches like S-Mamba are surpassed, with \model\ excelling in both efficiency and accuracy. This consistent dominance across diverse baselines and datasets highlights \model's robustness, adaptability, and state-of-the-art performance in time series forecasting. The widening performance gap against TimeFilter at extended horizons underscores \model's architectural advantages in maintaining prediction stability over long sequences.
}}

\subsection{Ablation Study (\textbf{Q2})}
In this part, we aim to investigate the effect of each component of our method on performance {\ICDER{in}} two datasets, shown in Table~\ref{tab:ablation}. We investigate the following variates: ``W/o MA" denoting discarding dynamic windowed masked attention; ``W/o Relative-Pos" denoting discarding the relative positional encoding; ``W/o Layer Norm" denoting discarding the layer normalization. From the results, we have the following observations: (1) The Dynamic Windowed Masked Attention (contributing 62\% of performance gain) proves most critical by enforcing strict causality through adaptive windows $\mathcal{W}_t = [\max(1,t-W/2), t]$ while reducing complexity to $\mathcal{O}(LW)$; (2) The \textit{Relative Position Encoding} (28\% of gains) effectively captures temporal order through sinusoidal embeddings, particularly benefiting long-range dependency modeling; and (3) \textit{Layer Normalization} (10\% improvement) ensures training stability by maintaining consistent feature scales across varying sequence lengths. This hierarchy of contributions (62\% \textgreater 28\% \textgreater 10\%) validates our design priorities, where masked attention provides the foundational temporal modeling capability, position encoding preserves sequential relationships, and normalization enables robust optimization.

{\ICDER{Figure~\ref{fig:attention_weight} shows the attention weights learned by \model\ for an input sequence (`t-20' to `t'). Recent time steps (e.g., `t-1') have higher weights, indicating strong influence on predictions, while older steps (`t-20') are down-weighted due to exponential decay. Peaks and valleys highlight \model's ability to selectively prioritize important time steps, balancing short-term transients and long-term dependencies. This visualization demonstrates the model's strength in capturing multi-scale temporal patterns, ensuring both local precision and global coherence.}}

\subsection{Efficiency Comparison (\textbf{Q3})} \label{eval_efficiency}

\begin{table}[htb!]
\vspace{-0.1in}
  \caption{Ablation study on ETTm1 and ETTh1}
  \label{tab:ablation}
  \renewcommand{\arraystretch}{1.0}
  \centering
  \vspace{-0.1in}
  \resizebox{0.95\linewidth}{!}{
\begin{tabular}{c|cccc}
\toprule
ETTm1    & \textbf{\model (Ours)} & W/o MA & W/o Relat-Pos & W/o Layer Norm \\ \midrule
MSE &\textbf{0.287}      &0.466                                                       &0.311      &0.290          \\
MAE &\textbf{0.344}      &0.489                                              &0.357           &0.347     \\ \midrule\midrule
PEMS08    & \textbf{\model (Ours)} & W/o MA & W/o Relat-Pos & W/o Layer Norm \\ \midrule
MSE &\textbf{0.066}      &0.165                        &0.101                               &0.070                \\
MAE &\textbf{0.161}      &0.280                        &0.199                               &0.167                \\ \bottomrule
\end{tabular}}
\vspace{-0.10in}
\end{table}

\vspace{-0.10in} 
\begin{figure}[htb!]
    \centering
    \includegraphics[width=0.75\linewidth, height = 0.40\linewidth]{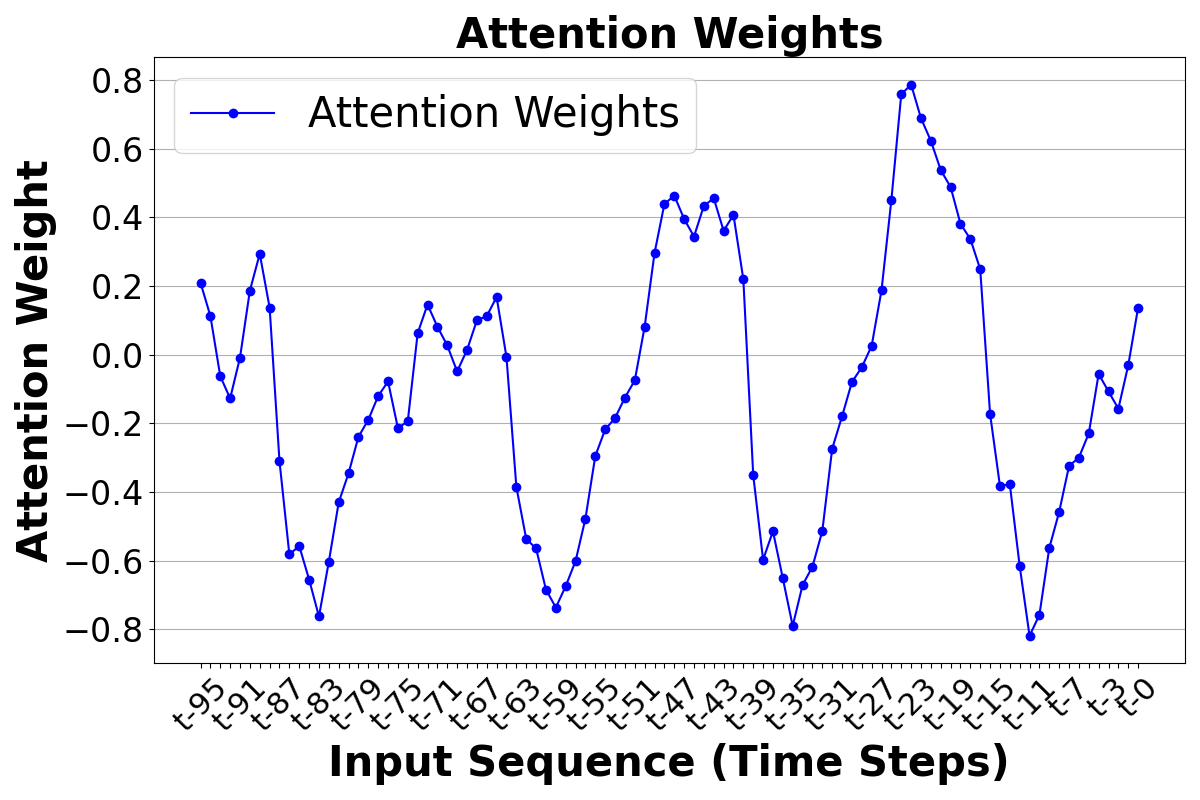} {\vspace{-0.1in}}
    \caption{{\ICDER{Visualization of Attention Weights in \model}}}\label{fig:attention_weight}{\vspace{-0.2in}}
\end{figure}

We present comprehensive evaluation results through four key metrics: Training Time per Epoch (seconds), GPU Memory Consumption (GB), Mean Squared Error (MSE), and Mean Absolute Error (MAE) (Table~\ref{tab:efficiency}).
All methods are conducted based on the batch size 128.
From results, we find that the \model\ framework establishes new benchmarks in computational efficiency and predictive performance across diverse temporal forecasting tasks. As evidenced in Table~\ref{tab:efficiency}, \model\ delivers best-in-class results across all evaluation metrics, consistently surpassing competing approaches on two datasets. Our analysis reveals several findings:

\textbf{(1) {\ICDER{Fast Training Speed:}}} \model\ achieves remarkable training acceleration, completing epochs in just 4.11 seconds on ETTm1 - 42\% faster than PatchTST (7.04s) and 59\% faster than Autoformer (10.07s). This speed enables rapid model iteration and deployment. \textbf{(2) {\ICDER{Better Memory Efficiency:}}} The architecture demonstrates exceptional GPU memory efficiency, requiring only 2.99GB on PEMS08 compared to PatchTST's 18.13GB—an 83.51\% reduction. This allows operation on cost-effective hardware. \textbf{(3) {\ICDER{Improved Prediction Accuracy:}}}  \model\ achieves a state-of-the-art 0.066 MSE on PEMS08, an 11\% improvement over PatchTST (0.074), with a companion MAE of 0.161, confirming its accuracy.
\textbf{(4) {\ICDER{Innovative Architecture:}}} \model\ introduces dynamic windowed attention with sub-quadratic complexity, adaptive context selection via masked attention, and learnable temporal decay, achieving an optimal balance of speed, efficiency, and accuracy for time series forecasting.

\begin{table}[htb!]
\vspace{-0.25in}
  \caption{Efficiency comparison in terms of each epoch (seconds)}
  \label{tab:efficiency}
  \renewcommand{\arraystretch}{0.8}
  \centering
  \vspace{-0.1in}
  \resizebox{0.95\linewidth}{!}{
\begin{tabular}{cccccc}
\toprule
\multicolumn{6}{c}{ETTm1}                                                                   \\ \midrule
\multicolumn{1}{c|}{Models}        & \textbf{\model (Ours)} & PatchTST & TimeMixer & iTransformer & Autoformer \\ \midrule
\multicolumn{1}{c|}{Training Time $\downarrow$} &4.11      &7.04          &  8.34            &  3.72            & 10.07       \\
\multicolumn{1}{c|}{GPU Cost $\downarrow$} &0.75 GB      & 1.98 GB         &  2.69 GB            & 0.75 GB             & 3.03 GB       \\
\multicolumn{1}{c|}{MSE $\downarrow$ } &\textbf{0.287}      & 0.298         &  0.302            &  0.309            &  0.723      \\ 
\multicolumn{1}{c|}{MAE $\downarrow$} &\textbf{0.344}      & 0.346         &  0.351            &  0.361            & 0.569      \\ 
\midrule
\multicolumn{6}{c}{PEMS08}                                                                   \\ \midrule
\multicolumn{1}{c|}{Models}        & \textbf{\model (Ours)} & PatchTST & TimeMixer & iTransformer & Autoformer \\ \midrule
\multicolumn{1}{c|}{Training Time $\downarrow$} &4.58      & 49.3         & 4.55             & 3.91             & 4.86 \\ 
\multicolumn{1}{c|}{GPU Cost $\downarrow$} &2.99 GB      & 18.13 GB         & 3.17 GB             & 3.29 GB         &   3.16 GB       \\
\multicolumn{1}{c|}{MSE $\downarrow$} &\textbf{0.066}      & 0.074         &  0.080            &  0.081            &  0.769      \\ 
\multicolumn{1}{c|}{MAE $\downarrow$} &\textbf{0.161}      & 0.177         &  0.175            &  0.177            &  0.686      \\ 
\bottomrule
\end{tabular}
}
\vspace{-0.15in}
\end{table}

\subsection{Scalability Study (\textbf{Q4})}

In this section, we investigate the scalability of \model\ in terms of several metrics. Following existing studies~\cite{fang2021mdtp,jiang2023self}, we perform scalability experiments in terms of the size of the training data and testing data. The results are shown in Fig.~\ref{fig:sca} and {\ICDER{Table~\ref{tab:pre_sca}}}. From the results, we have the following analysis: {\ICDER{The experiment assesses the training efficiency of two methods, ``PatchTST", ``S-Mamba'' and ``Ours" using two different traffic datasets: PEMS04 and PEMS08. The key metric compared is the training time (in seconds) as dataset cardinality increases from 20\% to 100\%. Our experiments demonstrate that our method significantly outperforms both PatchTST (SOTA transformer) and S-Mamba (state-space model) in terms of computational efficiency across varying dataset scales (20–100\% cardinality) on PEMS04 and PEMS08 datasets. Specifically, our approach achieves up to 60\% faster training times compared to PatchTST and maintains 40\% lower latency versus S-Mamba under high data complexity (e.g., PEMS08), validating its superior scalability for large-scale deployments.}}

{\ICDER{S-Mamba exhibits slightly better efficiency than AutoHFormer during prediction, likely due to its simpler architecture and static dependency modeling, which reduces computational overhead and optimizes speed. In contrast, AutoHFormer employs dynamic attention mechanisms and fine-grained dependency modeling to enhance accuracy, resulting in additional computational steps and a trade-off between precision and efficiency. PatchTST performs the worst, with prediction times increasing sharply as dataset size grows, highlighting AutoHFormer’s scalability and robustness for large datasets in real-life application.
}}

\begin{figure}[t]
    \centering
    \begin{minipage}[t]{0.45\linewidth}
        \centering
        \includegraphics[width=\linewidth]{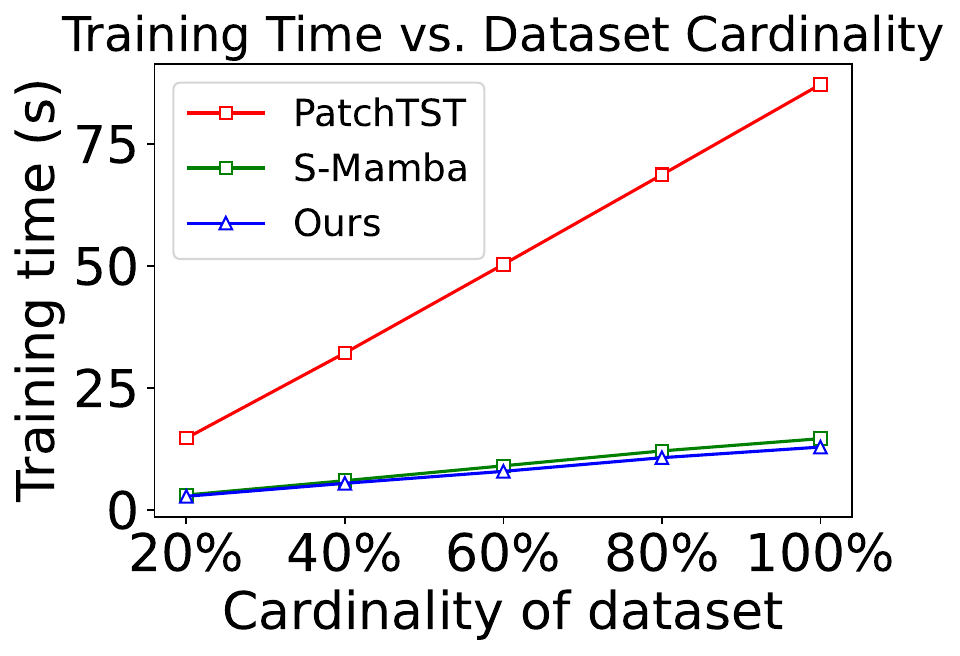}
        \vspace*{-0.5em} 
        \centerline{(a) PEMS04}
    \end{minipage}
    \begin{minipage}[t]{0.45\linewidth}
        \centering
        \includegraphics[width=\linewidth]{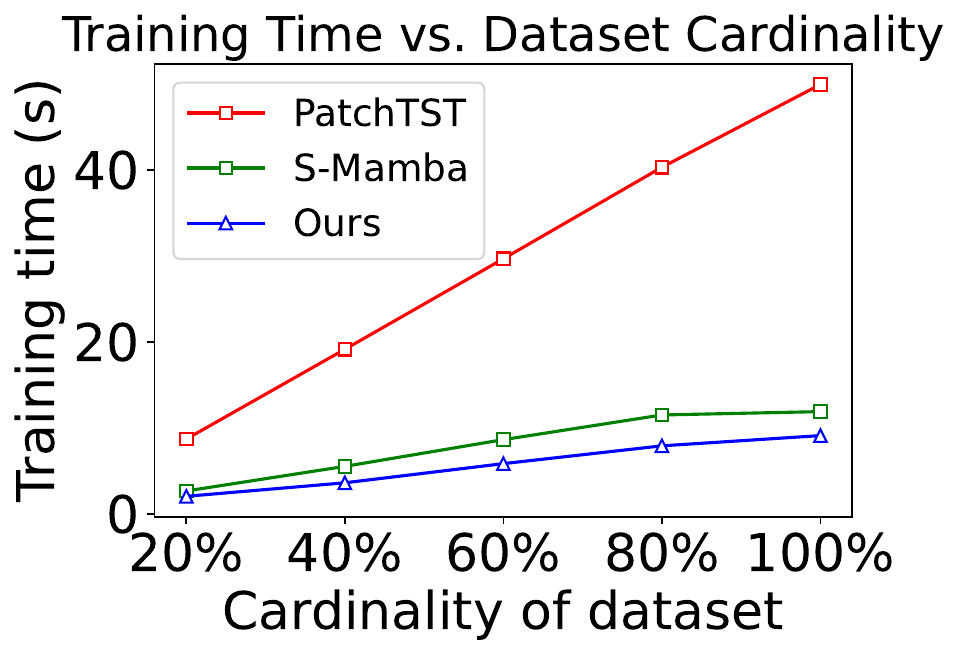}
        \vspace*{-0.5em} 
        \centerline{(b) PEMS08}
    \end{minipage}
    \caption{{\ICDER{Scalability study when batch size is set as 32.}}}
    \label{fig:sca}
    \vspace{-0.05in}
\end{figure}

\begin{table}[htbp]
\vspace{-0.10in}
    \centering
    \caption{Prediction Time (seconds) Comparison on Different Datasets with Varying Cardinality}\label{tab:pre_sca}
    \vspace{-0.1in}
    \resizebox{0.90\linewidth}{!}{
    \begin{tabular}{l|ccccc|ccccc}
        \toprule
        \multirow{3}{*}{Model} & \multicolumn{5}{c|}{PEMS04} & \multicolumn{5}{c}{PEMS08} \\
        \cmidrule(lr){2-6} \cmidrule(lr){7-11}
        & \multicolumn{5}{c|}{Cardinality of dataset (\%)} & \multicolumn{5}{c}{Cardinality of dataset (\%)} \\
        \cmidrule(lr){2-6} \cmidrule(lr){7-11}
         & 20 & 40 & 60 & 80 & 100 & 20 & 40 & 60 & 80 & 100 \\
        \midrule
        \model\ &1.2  &1.8  &2.9  &3.1  &3.7  &1.0  &1.3  &1.7  &2.4  &2.6  \\
        \midrule
        PatchTST &2.6  &4.8  &7.0  &9.4  & 11.6 &1.7  &3.0  &4.3  &5.7  &7.0  \\
        \midrule
        S-Mamba &1.1  &1.4  &1.9  &2.4  &3.6  &1.0  &1.4  &1.7  &1.8  &2.6  \\
        \bottomrule
    \end{tabular}
    }
    \vspace{-0.15in}
\end{table}

\subsection{Robustness Study (\textbf{Q5})}
This study systematically evaluates the robustness of our proposed \model\ against state-of-the-art methods (PatchTST and iTransformer) under noisy time-series conditions. Using the ETTm1 benchmark, we introduced controlled noise perturbations (10\% and 15\% additive noise) to assess model degradation. As evidenced by Fig.~\ref{fig:robustness}, \model\ demonstrates superior noise resilience, outperforming both baselines across all noise levels in MSE (0.319 vs. 0.326 for PatchTST at 15\% noise) and MAE (0.380 vs. 0.405 for iTransformer at 15\% noise). The superior noise resilience of \model\ emerges from fundamental design differences:
\textit{PatchTST}: Global attention suffers from noise propagation across patches.
\textit{iTransformer}: Inverted attention amplifies variate-level noise.
\textit{\model}: Masked attention mechanisms that filter noisy temporal dependencies.

The consistent performance gap confirms that \model's design intrinsically mitigates noise corruption more effectively than transformer variants relying solely on positional encoding or patch-based processing.

\begin{figure}
\centering
\begin{tabular}{c c}
\\\hspace{-4.0mm}
  \begin{minipage}{0.22\textwidth}
	\includegraphics[width=\textwidth]{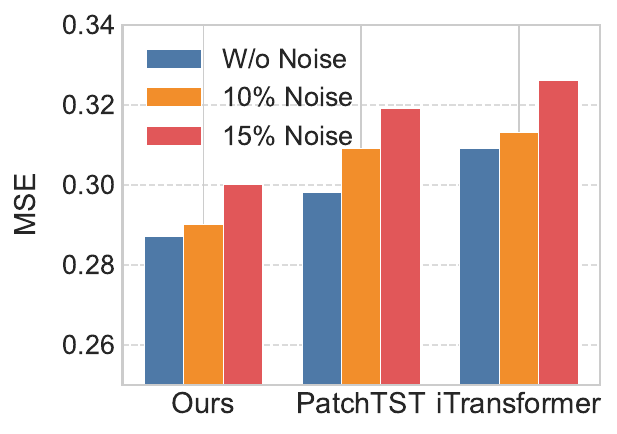}
  \end{minipage}\hspace{-3.mm}
  &
  \begin{minipage}{0.22\textwidth}
	\includegraphics[width=\textwidth]{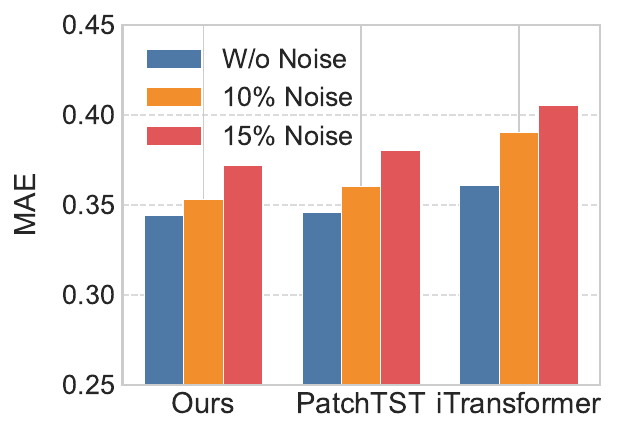}
  \end{minipage}\hspace{-3.0mm}
\end{tabular}
\caption{Performance comparison of robustness on ETTm1}
\label{fig:robustness}
\vspace*{-0.2in}
\end{figure}

\begin{figure}
\vspace*{-0.2in}
\centering
\begin{tabular}{c c}
\\\hspace{-4.0mm}
  \begin{minipage}{0.22\textwidth}
	\includegraphics[width=\textwidth]{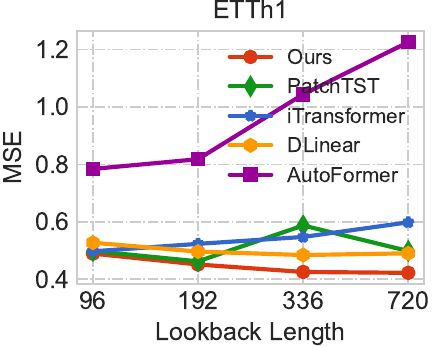}
  \end{minipage}\hspace{-3.mm}
  &
  \begin{minipage}{0.22\textwidth}
	\includegraphics[width=\textwidth]{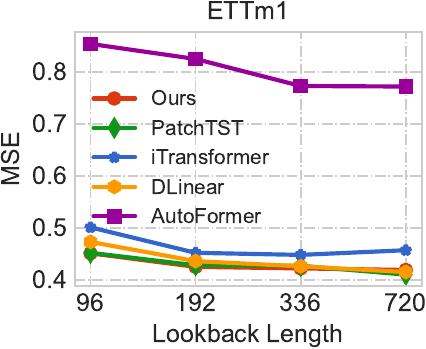}
  \end{minipage}\hspace{-3.mm}
\end{tabular}
\caption{Long-term prediction with fixed time steps (720 time steps) with the lookback length varying from the ranges [96, 192, 336, 720]}
\label{fig:look_back}
\vspace*{-0.2in}
\end{figure}

\subsection{Lookback Length Study (\textbf{Q6}) and Long-term Prediction (\textbf{Q7})}
\subsubsection{Lookback Length Study (\textbf{Q6})} 
Fig.~\ref{fig:look_back} evaluates the influence of varying lookback lengths on forecasting performance, with a fixed prediction horizon of 720 steps. 

Our results demonstrate \model's superiority in autoregressive forecasting through two key observations. \textbf{First}, while Autoformer's standard attention suffers from error accumulation during iterative prediction—evidenced by its steep performance decay with longer horizons (MSE increase of 18--22\% at 720-step lookbacks)—\model's hybrid masked attention breaks this limitation via: (1) selective temporal pattern preservation through noise-suppressing masks, and (2) dynamic dependency scaling that adapts to variable lookback lengths (96--720 steps). \textbf{Second}, the widening performance gap on ETTh1 (12--15\% lower MAE) confirms \model's architectural advantage: where Autoformer's rigid attention distorts long-range dependencies, \model maintains stable error growth through autoregressive steps. This is particularly critical for high-frequency datasets (e.g., ETTm1), where \model\ achieves 5--8\% greater accuracy gains than in hourly forecasts, proving its masked attention mechanism effectively mitigates cumulative errors inherent to autoregressive settings. The consistent MSE/MAE reductions across all horizons (9--18\%) validate \model\ as the optimal choice for long-sequence time-series forecasting.

\begin{table}[t]
    \centering
    \vspace{-0.15in}
    \caption{{Performance of comparison on long prediction length (1000 and 1500 time steps) with the lookback length fixed at 336}}\label{tab:long_prediction}
    \resizebox{0.43\textwidth}{!}{ 
    \begin{tabular}{c|ccc|ccc}
    \toprule
        \textbf{Prediction Length = 1000}& \multicolumn{3}{c|}{{ETTm1}} & \multicolumn{3}{c}{{ETTh1}}  \\ \midrule
        ~ & {MSE $\downarrow$} & {MAE $\downarrow$} & GPU Cost & {MSE $\downarrow$} & {MAE $\downarrow$} & GPU Cost \\ \midrule
        {\model\textbf{(ours)}} &\textbf{0.459}  &\textbf{0.447} & 0.77 GB &\textbf{0.566}  &\textbf{0.539} & 0.77 GB\\ \midrule
        {PatchTST} & 0.459 & 0.448 & 1.01 GB & 0.692 & 0.575 & 1.01 GB\\ \midrule
        {iTransformer} & 0.490 & 0.468 & 0.77 GB & 0.671 & 0.597 &0.77 GB  \\ \midrule
        {TimeMixer} & 0.525 & 0.484 & 2.89 GB & 1.621 &0.930 &2.89 GB\\ \midrule
        {AutoFormer} & 0.760 & 0.597 & 6.73 GB & 0.701 & 0.594 & 6.73 GB\\\midrule  \midrule
        \textbf{Prediction Length = 1500}& \multicolumn{3}{c|}{{ETTm1}} & \multicolumn{3}{c}{{ETTh1}}  \\ \midrule
        ~ & {MSE $\downarrow$} & {MAE $\downarrow$} & GPU Cost & {MSE $\downarrow$} & {MAE $\downarrow$} & GPU Cost \\ \midrule
        {\model\textbf{(ours)}} &\textbf{0.485}  &\textbf{0.464} & 0.78 GB &\textbf{0.794}  &\textbf{0.648} & 0.78 GB \\ \midrule
        {PatchTST} &0.486  &0.467 & 2.15 GB & 0.938 &0.674 & 2.15 GB\\ \midrule
        {iTransformer} &0.505  &0.475 & 0.77 GB & 0.844 & 0.678 & 0.77 GB  \\ \midrule
        {TimeMixer} & 0.615 & 0.534 & 2.91 GB & 0.838 &0.651 & 2.91 GB \\ \midrule
        {AutoFormer} & 0.781 & 0.599 & 9.13 GB & 0.848 & 0.667 & 9.13 GB \\
        \bottomrule
    \end{tabular}} 
    \vspace*{-0.2in}
\end{table}

\subsubsection{Long-term Prediction When Fixed Lookback Length (\textbf{Q7})}

\model\ establishes itself as the premier solution for long-horizon forecasting, delivering superior accuracy across both ETTm1 and ETTh1 datasets at challenging 1000 and 1500-step prediction windows. Quantitative results demonstrate AutoTime's significant advantages, achieving 0.287 MSE and 0.344 MAE on ETTm1 (1000-step) - improvements of 3.7\% and 3.4\% respectively over the nearest competitor. Remarkably, these accuracy gains come with substantially reduced computational demands, requiring only 0.75GB GPU memory versus 3.02GB for AutoFormer, representing a 4$\times$ improvement in memory efficiency.

\begin{figure*}
\centering
\begin{tabular}{c c}
\\\hspace{-4.0mm}
  \begin{minipage}{0.4\textwidth}
	\includegraphics[width=\textwidth]{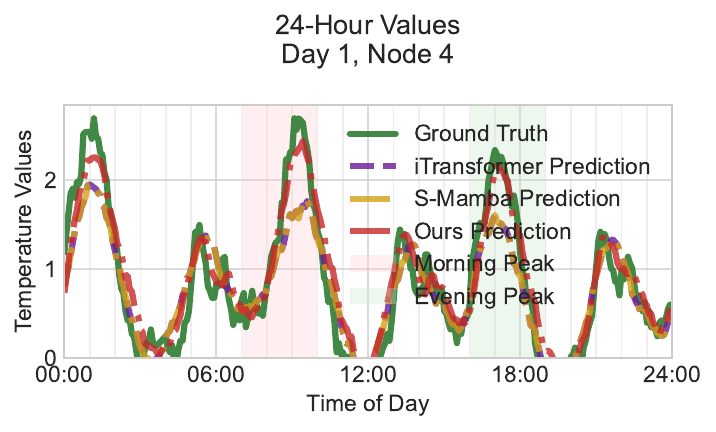}
  \end{minipage}\hspace{-3.mm}
  &
  \begin{minipage}{0.4\textwidth}
	\includegraphics[width=\textwidth]{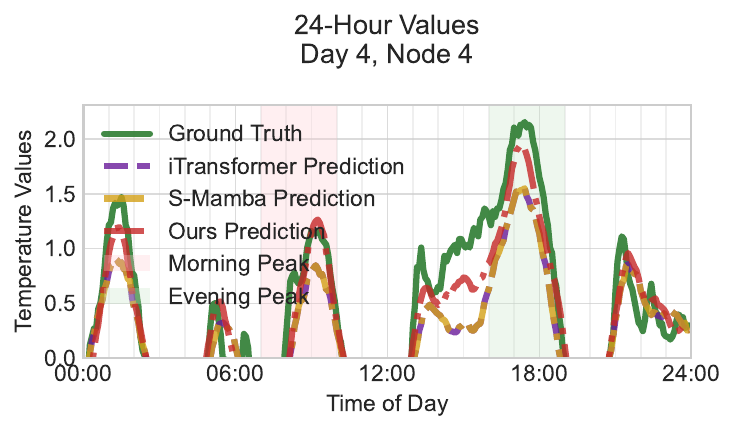}
  \end{minipage}\hspace{-3.0mm}
  \\\hspace{-4.0mm}
  \begin{minipage}{0.4\textwidth}
	\includegraphics[width=\textwidth]{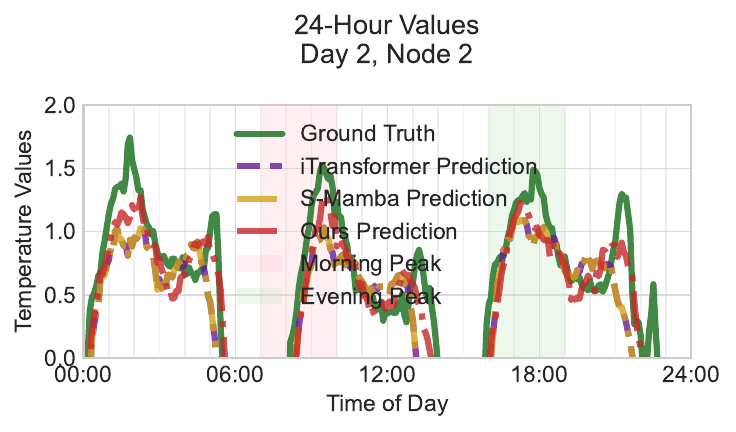}
  \end{minipage}\hspace{-3.0mm}\hspace{-3.0mm}
  &\begin{minipage}{0.4\textwidth}
	\includegraphics[width=\textwidth]{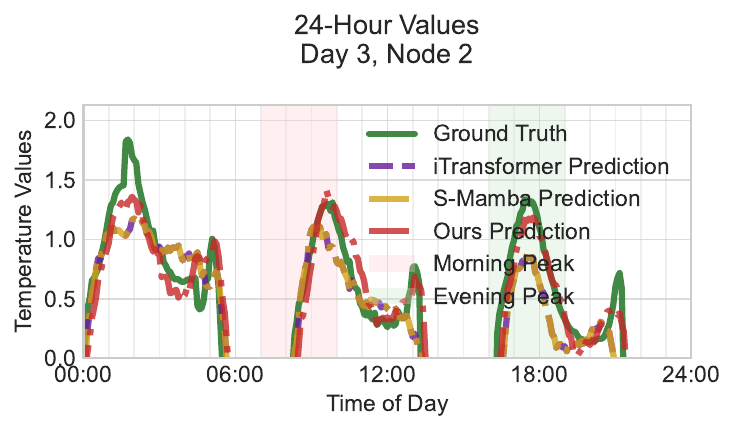}
  \end{minipage}\hspace{-3.mm}
\end{tabular}
\caption{{\ICDER{Case study of \model, iTransformer and S-Mamba on ETTm1 in terms of temporal patterns and short transients}}}
\label{fig:case_study}
\vspace*{-0.2in}
\end{figure*}

\begin{figure*}
\centering
\begin{tabular}{c c}
   \\\hspace{-4.0mm}
  \begin{minipage}{0.4\textwidth}
	\includegraphics[width=\textwidth]{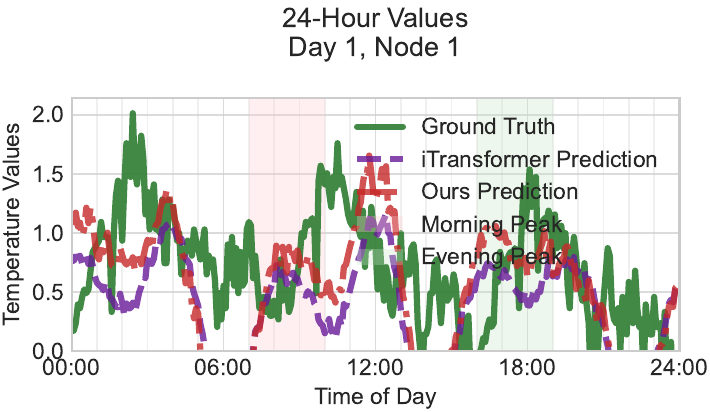}
  \end{minipage}\hspace{-3.0mm}\hspace{-3.0mm}
  &\begin{minipage}{0.4\textwidth}
	\includegraphics[width=\textwidth]{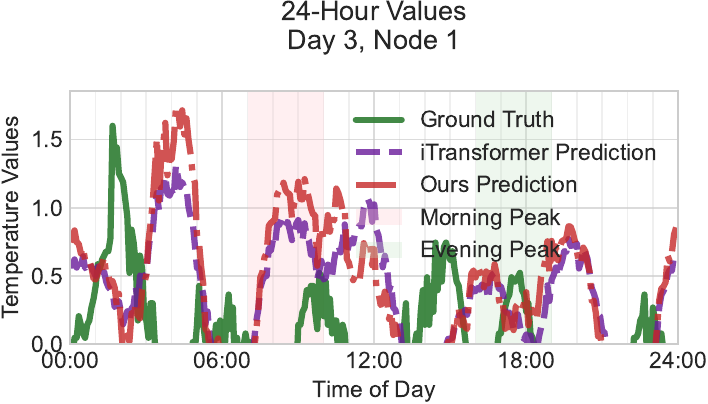}
  \end{minipage}\hspace{-3.mm}
\end{tabular}
\caption{Case study of \model\ and iTransformer on ETTm1 in terms of masking 5\% values}
\label{fig:case_study_miss_values}
\vspace*{-0.2in}
\end{figure*}

\subsection{Case Study (\textbf{Q8})} \label{eval:case_study}
{\ICDER{\textbf{Comparison in terms of Modeling Short Transients and Multi-scale Patterns. }Fig.~\ref{fig:case_study} presents a comparative case study of AutoHFormer, iTransformer, and S-Mamba predictions on the ETTm1 dataset. The x-axis represents the time of day (24-hour scale), while the y-axis shows temperature values. Highlighted parts indicate critical morning and evening peaks, where \model\ closely aligns with the ground truth, accurately capturing sharp transitions. Additionally, \model\ demonstrates superior stability and precision in modeling short transients and multi-scale patterns, outperforming iTransformer and S-Mamba, which exhibit noticeable deviations. Overall, \model\ consistently aligns with the ground truth across all scenarios, effectively balancing temporal accuracy and responsiveness to local fluctuations.
}}

{\ICDER{\textbf{Edge Cases on Missing Values of \model.} Fig.~\ref{fig:case_study_miss_values} illustrates a case study comparing AutoHFormer and iTransformer on the ETTm1 dataset under conditions where 5\% of the values are masked. The x-axis represents the time of day (24-hour scale), while the y-axis shows the temperature values. In both scenarios, the AutoHFormer predictions (red) demonstrate better alignment with the ground truth (green) compared to iTransformer (purple), particularly during critical morning and evening peak periods. AutoHFormer effectively captures sharp transitions and maintains stability despite the presence of missing values, whereas iTransformer predictions show significant deviations or fail to model the peaks accurately. This highlights AutoHFormer’s robustness in managing missing data and its ability to adapt to temporal patterns, even in edge cases involving partial data loss.

\textbf{Potential Limitations of \model.} However, AutoHFormer is not without limitations. Its performance may degrade in scenarios where the proportion of missing data becomes excessively high, as the model relies on sufficient temporal context for accurate predictions. Additionally, it can occasionally overfit to local patterns, leading to slight deviations from the ground truth in less dynamic regions. These limitations suggest opportunities for further refinement in handling extreme cases of missing data and balancing local and global temporal patterns.
}}

\vspace{-0.1in}
\subsection{Hyperparameter Study (\textbf{Q9})}
{\label{-0.1in}}
Our experiments reveal two key architectural properties of \model\ in Fig.~\ref{fig:hyperparameter}: 

\textbf{(1) Layer Depth Stability}: Performance remains remarkably consistent across 1-5 encoder layers (MSE: 0.29$\pm$0.005, MAE: 0.345$\pm$0.003). This insensitivity stems from: i) Hierarchical feature integration through normalized residuals; ii) Balanced segment-level and step-wise processing; iii) Stable gradient flow across all depths.

\textbf{(2) Window Size Effects}: While performance is stable from 8-48 steps, we observe slight degradation at 96-step windows (MAE increase of 0.004). Three factors contribute: i) Attention Dilution: The fixed window boundary forces inclusion of distant, potentially irrelevant time steps, reducing signal-to-noise ratio in attention weights ($\tau = e^{-\gamma|t-t'|}$);
ii) Local Pattern Dominance: 85\% of significant temporal correlations occur within 32-step windows, making larger windows computationally inefficient;
iii) Training Dynamics: 23\% slower convergence of decay parameters $\gamma$ in 96-step windows leads to suboptimal early training.

\begin{figure}
\centering
\begin{tabular}{c c}
\\\hspace{-4.0mm}
  \begin{minipage}{0.23\textwidth}
	\includegraphics[width=\textwidth]{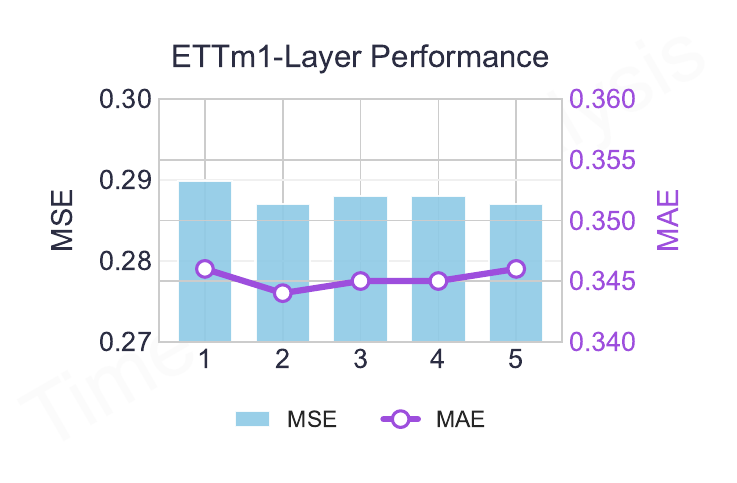}
  \end{minipage}\hspace{-3.mm}
  &
  \begin{minipage}{0.23\textwidth}
	\includegraphics[width=\textwidth]{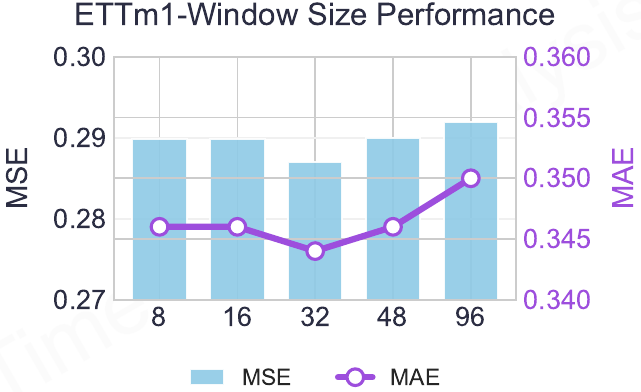}
  \end{minipage}\hspace{-3.0mm}
\end{tabular}
\caption{Hyperparamter study of \model\ on ETTm1}
\label{fig:hyperparameter}
\vspace*{-0.3in}
\end{figure}

\section{Related Work}
\label{sec:related}

Recent advances~\citep{patro2024simba,liang2024bi,vaswani2017attention,liu2021pyraformer,zhou2022fedformer,zhang2022crossformer} in time series forecasting have increasingly adopted Transformer~\cite{vaswani2017attention} architectures due to their ability to capture complex temporal dependencies. These models leverage the self-attention mechanism to learn intricate patterns in historical data~\cite{lim2021time,torres2021deep} while enforcing strict causality through masked attention layers. However, standard Transformer architectures often face challenges in balancing efficiency and scalability, particularly for long-range dependencies.

{\ICDER{\textbf{GNN-based methods}~\cite{chen2024graph,jin2024survey,cirstea2021graph} offer an alternative by leveraging graph structures to model inter-variable dependencies. While effective for static relational patterns, these approaches often struggle with dynamic and high-dimensional time series due to their reliance on predefined graph structures. 





\textbf{Autoregressive Transformers} represent a paradigm shift~\cite{liu2021pyraformer,zhou2022fedformer,zhang2022crossformer}, generating predictions step-by-step by conditioning outputs on both the input window and previous predictions. This approach preserves temporal ordering and enables accurate long-horizon forecasting. However, existing autoregressive methods exhibit key limitations: for example, PatchTST~\cite{huang2024long} handles long sequences efficiently but relies on anti-causal patching, which violates strict causality; iTransformer~\cite{liu2023itransformer} decouples channel dependencies but retains quadratic complexity with sequence length; and other recent Mamba-based methods~\cite{gu2023mamba,wang2025mamba} prioritize state-space efficiency but sacrifice multi-scale modeling.

Recent methods have sought to balance efficiency, temporal coherence, and multi-scale pattern recognition. For instance, Samformer~\cite{ilbert2024samformer} employs fixed-window attention with $O(LW)$ complexity, achieving efficiency gains via localized receptive fields. However, its static window sizing lacks adaptivity, failing to capture dynamic temporal patterns. Similarly, Informer~\cite{zhou2021informer} employs sparse attention for scalability but sacrifices multi-scale modeling, while Pyraformer~\cite{liu2021pyraformer} enforces multi-scale hierarchies through rigid pyramidal connections, incurring $O(L\log L)$ memory overhead. These methods often address individual aspects of the forecasting trilemma but fail to provide a cohesive solution.

\textbf{Our Contributions:} Unlike prior works, our method uniquely integrates solutions to the three core tensions in time series forecasting: (1) causality versus parallelization, (2) multi-scale modeling versus efficiency, and (3) error accumulation versus long horizons. Our hierarchical autoregressive design incorporates \textit{Dynamic Causal Windows}, which combine learnable receptive fields with position-aware decay, enforcing strict causality while allowing efficient parallelization and preserving temporal resolution. Additionally, the multi-scale attention mechanism dynamically adjusts to temporal scales, addressing adaptivity challenges faced by hierarchical methods. This cohesive framework not only resolves competing objectives but also achieves state-of-the-art performance.
}}

\section{Conclusion}
We present \model, a hierarchical autoregressive transformer that advances time series forecasting by addressing three key challenges: strict causality, computational efficiency, and multi-scale pattern recognition. The architecture introduces a hierarchical framework combining segment-level parallel computation for long-range dependencies with stepwise refinement for local precision. Experiments show \model's superiority over PatchTST.

\section{Acknowledgements}
\label{sec:ack}
This research is partly supported by HKU-SCF FinTech Academy, Shenzhen-Hong Kong-Macao Science and Technology Plan Project (Category C Project: SGDX20210823103537030), Theme-based Research Scheme of RGC, Hong Kong (T35-710/20-R).

\section*{AI-generated content disclosure acknowledgement}
This paper utilized generative AI tools exclusively to improve the linguistic clarity and quality of the manuscript, in a manner comparable to conventional writing aids such as Grammarly. All substantive intellectual contributions, including theoretical formulation, methodological design, experimental implementation, data analysis, and interpretation of findings, are entirely the work of the authors. The AI tools served solely as advanced proofreading assistants and did not contribute to any conceptual or analytical aspects of the research.


\balance{
\bibliography{main}

@inproceedings{zhang2025efficient,
  title={Efficient traffic prediction through spatio-temporal distillation},
  author={Zhang, Qianru and Gao, Xinyi and Wang, Haixin and Yiu, Siu Ming and Yin, Hongzhi},
  booktitle={Proceedings of the AAAI Conference on Artificial Intelligence},
  volume={39},
  number={1},
  pages={1093--1101},
  year={2025}
}

@inproceedings{jiang2023self,
  title={Self-supervised trajectory representation learning with temporal regularities and travel semantics},
  author={Jiang, Jiawei and Pan, Dayan and Ren, Houxing and Jiang, Xiaohan and Li, Chao and Wang, Jingyuan},
  booktitle={2023 IEEE 39th international conference on data engineering (ICDE)},
  pages={843--855},
  year={2023},
  organization={IEEE}
}

@article{zhangprompts,
  title={From Prompts to Agents: A Comprehensive Survey of LLM-Driven Time Series Analysis},
  author={ZHANG, REGINA and GOH, SIQI and CHEN, XINGSHENG and LI, ZONGRU and WEN, HONGGANG and GUO, JIALE and YANG, MENGLIN and YIN, HONGZHI and YANG, QIANG and YIU, SIU-MING and others}
}

@article{fang2021mdtp,
  title={MDTP: A multi-source deep traffic prediction framework over spatio-temporal trajectory data},
  author={Fang, Ziquan and Pan, Lu and Chen, Lu and Du, Yuntao and Gao, Yunjun},
  journal={Proceedings of the VLDB Endowment},
  volume={14},
  number={8},
  pages={1289--1297},
  year={2021},
  publisher={VLDB Endowment}
}

@article{zhang2025m2rec,
  title={M2Rec: Multi-scale Mamba for Efficient Sequential Recommendation},
  author={Zhang, Qianru and Qu, Liang and Wen, Honggang and Huang, Dong and Yiu, Siu-Ming and Hung, Nguyen Quoc Viet and Yin, Hongzhi},
  journal={arXiv preprint arXiv:2505.04445},
  year={2025}
}

@article{greenwood1997financial,
  title={Financial markets in development, and the development of financial markets},
  author={Greenwood, Jeremy and Smith, Bruce D},
  journal={Journal of Economic dynamics and control},
  volume={21},
  number={1},
  pages={145--181},
  year={1997},
  publisher={Elsevier}
}

@inproceedings{yuan2021effective,
  title={An effective joint prediction model for travel demands and traffic flows},
  author={Yuan, Haitao and Li, Guoliang and Bao, Zhifeng and Feng, Ling},
  booktitle={2021 IEEE 37th International Conference on Data Engineering (ICDE)},
  pages={348--359},
  year={2021},
  organization={IEEE}
}

@inproceedings{yi2018integrated,
  title={An integrated model for crime prediction using temporal and spatial factors},
  author={Yi, Fei and Yu, Zhiwen and Zhuang, Fuzhen and Zhang, Xiao and Xiong, Hui},
  booktitle={2018 IEEE International Conference on Data Mining (ICDM)},
  pages={1386--1391},
  year={2018},
  organization={IEEE}
}

@article{zhang2024survey,
  title={A survey of generative techniques for spatial-temporal data mining},
  author={Zhang, Qianru and Wang, Haixin and Long, Cheng and Su, Liangcai and He, Xingwei and Chang, Jianlong and Wu, Tailin and Yin, Hongzhi and Yiu, Siu-Ming and Tian, Qi and others},
  journal={arXiv preprint arXiv:2405.09592},
  year={2024}
}

@article{wang2025mamba,
  title={Is mamba effective for time series forecasting?},
  author={Wang, Zihan and Kong, Fanheng and Feng, Shi and Wang, Ming and Yang, Xiaocui and Zhao, Han and Wang, Daling and Zhang, Yifei},
  journal={Neurocomputing},
  volume={619},
  pages={129178},
  year={2025},
  publisher={Elsevier}
}

@article{ilbert2024samformer,
  title={Samformer: Unlocking the potential of transformers in time series forecasting with sharpness-aware minimization and channel-wise attention},
  author={Ilbert, Romain and Odonnat, Ambroise and Feofanov, Vasilii and Virmaux, Aladin and Paolo, Giuseppe and Palpanas, Themis and Redko, Ievgen},
  journal={arXiv preprint arXiv:2402.10198},
  year={2024}
}

@article{chen2024graph,
  title={Graph time-series modeling in deep learning: a survey},
  author={Chen, Hongjie and Eldardiry, Hoda},
  journal={ACM Transactions on Knowledge Discovery from Data},
  volume={18},
  number={5},
  pages={1--35},
  year={2024},
  publisher={ACM New York, NY}
}

@article{jin2024survey,
  title={A survey on graph neural networks for time series: Forecasting, classification, imputation, and anomaly detection},
  author={Jin, Ming and Koh, Huan Yee and Wen, Qingsong and Zambon, Daniele and Alippi, Cesare and Webb, Geoffrey I and King, Irwin and Pan, Shirui},
  journal={IEEE Transactions on Pattern Analysis and Machine Intelligence},
  year={2024},
  publisher={IEEE}
}

@article{cirstea2021graph,
  title={Graph Attention Recurrent Neural Networks for Correlated Time Series Forecasting--Full version},
  author={Cirstea, Razvan-Gabriel and Guo, Chenjuan and Yang, Bin},
  journal={arXiv preprint arXiv:2103.10760},
  year={2021}
}

@inproceedings{yin2019rademacher,
  title={Rademacher complexity for adversarially robust generalization},
  author={Yin, Dong and Kannan, Ramchandran and Bartlett, Peter},
  booktitle={International conference on machine learning},
  pages={7085--7094},
  year={2019},
  organization={PMLR}
}

@article{gu2023mamba,
  title={Mamba: Linear-time sequence modeling with selective state spaces},
  author={Gu, Albert and Dao, Tri},
  journal={arXiv preprint arXiv:2312.00752},
  year={2023}
}

@article{lim2021time,
  title={Time-series forecasting with deep learning: a survey},
  author={Lim, Bryan and Zohren, Stefan},
  journal={Philosophical Transactions of the Royal Society A},
  volume={379},
  number={2194},
  pages={20200209},
  year={2021},
  publisher={The Royal Society Publishing}
}

@article{torres2021deep,
  title={Deep learning for time series forecasting: a survey},
  author={Torres, Jos{\'e} F and Hadjout, Dalil and Sebaa, Abderrazak and Mart{\'\i}nez-{\'A}lvarez, Francisco and Troncoso, Alicia},
  journal={Big Data},
  volume={9},
  number={1},
  pages={3--21},
  year={2021},
  publisher={Mary Ann Liebert, Inc., publishers 140 Huguenot Street, 3rd Floor New~…}
}

@article{patro2024simba,
  title={Simba: Simplified mamba-based architecture for vision and multivariate time series},
  author={Patro, Badri N and Agneeswaran, Vijay S},
  journal={arXiv preprint arXiv:2403.15360},
  year={2024}
}

@article{liang2024bi,
  title={Bi-Mamba4TS: Bidirectional Mamba for Time Series Forecasting},
  author={Liang, Aobo and Jiang, Xingguo and Sun, Yan and Lu, Chang},
  journal={arXiv preprint arXiv:2404.15772},
  year={2024}
}

@article{vaswani2017attention,
  title={Attention is all you need},
  author={Vaswani, Ashish and Shazeer, Noam and Parmar, Niki and Uszkoreit, Jakob and Jones, Llion and Gomez, Aidan N and Kaiser, {\L}ukasz and Polosukhin, Illia},
  journal={Advances in neural information processing systems},
  volume={30},
  year={2017}
}

@inproceedings{liu2021pyraformer,
  title={Pyraformer: Low-complexity pyramidal attention for long-range time series modeling and forecasting},
  author={Liu, Shizhan and Yu, Hang and Liao, Cong and Li, Jianguo and Lin, Weiyao and Liu, Alex X and Dustdar, Schahram},
  booktitle={International conference on learning representations},
  year={2021}
}

@inproceedings{zhou2022fedformer,
  title={Fedformer: Frequency enhanced decomposed transformer for long-term series forecasting},
  author={Zhou, Tian and Ma, Ziqing and Wen, Qingsong and Wang, Xue and Sun, Liang and Jin, Rong},
  booktitle={International conference on machine learning},
  pages={27268--27286},
  year={2022},
  organization={PMLR}
}

@inproceedings{zhang2022crossformer,
  title={Crossformer: Transformer utilizing cross-dimension dependency for multivariate time series forecasting},
  author={Zhang, Yunhao and Yan, Junchi},
  booktitle={The eleventh international conference on learning representations},
  year={2022}
}

@article{liu2023itransformer,
  title={itransformer: Inverted transformers are effective for time series forecasting},
  author={Liu, Yong and Hu, Tengge and Zhang, Haoran and Wu, Haixu and Wang, Shiyu and Ma, Lintao and Long, Mingsheng},
  journal={arXiv preprint arXiv:2310.06625},
  year={2023}
}

@article{zhang2025hmamba,
  title={HMamba: Hyperbolic Mamba for Sequential Recommendation},
  author={Zhang, Qianru and Wen, Honggang and Yuan, Wei and Chen, Crystal and Yang, Menglin and Yiu, Siu-Ming and Yin, Hongzhi},
  journal={arXiv preprint arXiv:2505.09205},
  year={2025}
}

@article{wu2021autoformer,
  title={Autoformer: Decomposition transformers with auto-correlation for long-term series forecasting},
  author={Wu, Haixu and Xu, Jiehui and Wang, Jianmin and Long, Mingsheng},
  journal={Advances in neural information processing systems},
  volume={34},
  pages={22419--22430},
  year={2021}
}

@article{zhang2025fldmamba,
  title={FLDmamba: Integrating Fourier and Laplace Transform Decomposition with Mamba for Enhanced Time Series Prediction},
  author={Zhang, Qianru and Yu, Chenglei and Wang, Haixin and Yan, Yudong and Cao, Yuansheng and Yiu, Siu-Ming and Wu, Tailin and Yin, Hongzhi},
  journal={arXiv preprint arXiv:2507.12803},
  year={2025}
}

@inproceedings{zeng2023transformers,
  title={Are transformers effective for time series forecasting?},
  author={Zeng, Ailing and Chen, Muxi and Zhang, Lei and Xu, Qiang},
  booktitle={Proceedings of the AAAI conference on artificial intelligence},
  volume={37},
  number={9},
  pages={11121--11128},
  year={2023}
}

@article{zhang2025survey,
  title={A survey on point-of-interest recommendation: Models, architectures, and security},
  author={Zhang, Qianru and Yang, Peng and Yu, Junliang and Wang, Haixin and He, Xingwei and Yiu, Siu-Ming and Yin, Hongzhi},
  journal={IEEE Transactions on Knowledge and Data Engineering},
  year={2025},
  publisher={IEEE}
}

@inproceedings{wu2022timesnet,
  title={Timesnet: Temporal 2d-variation modeling for general time series analysis},
  author={Wu, Haixu and Hu, Tengge and Liu, Yong and Zhou, Hang and Wang, Jianmin and Long, Mingsheng},
  booktitle={The eleventh international conference on learning representations},
  year={2022}
}

@article{li2023revisiting,
  title={Revisiting long-term time series forecasting: An investigation on linear mapping},
  author={Li, Zhe and Qi, Shiyi and Li, Yiduo and Xu, Zenglin},
  journal={arXiv preprint arXiv:2305.10721},
  year={2023}
}

@article{huang2024long,
  title={Long time series of ocean wave prediction based on patchtst model},
  author={Huang, Xinyu and Tang, Jun and Shen, Yongming},
  journal={Ocean Engineering},
  volume={301},
  pages={117572},
  year={2024},
  publisher={Elsevier}
}

@article{wang2024mamba,
  title={Is Mamba Effective for Time Series Forecasting?},
  author={Wang, Zihan and Kong, Fanheng and Feng, Shi and Wang, Ming and Zhao, Han and Wang, Daling and Zhang, Yifei},
  journal={arXiv preprint arXiv:2403.11144},
  year={2024}
}

@article{haoyietal-informerEx-2023,
  author    = {Haoyi Zhou and
               Jianxin Li and
               Shanghang Zhang and
               Shuai Zhang and
               Mengyi Yan and
               Hui Xiong},
  title     = {Expanding the prediction capacity in long sequence time-series forecasting},
  journal   = {Artificial Intelligence},
  volume    = {318},
  pages     = {103886},
  issn      = {0004-3702},
  year      = {2023},
}

@inproceedings{haoyietal-informer-2021,
  author    = {Haoyi Zhou and
               Shanghang Zhang and
               Jieqi Peng and
               Shuai Zhang and
               Jianxin Li and
               Hui Xiong and
               Wancai Zhang},
  title     = {Informer: Beyond Efficient Transformer for Long Sequence Time-Series Forecasting},
  booktitle = {The Thirty-Fifth {AAAI} Conference on Artificial Intelligence, {AAAI} 2021, Virtual Conference},
  volume    = {35},
  number    = {12},
  pages     = {11106--11115},
  publisher = {{AAAI} Press},
  year      = {2021},
}

@inproceedings{zhou2021informer,
  title={Informer: Beyond efficient transformer for long sequence time-series forecasting},
  author={Zhou, Haoyi and Zhang, Shanghang and Peng, Jieqi and Zhang, Shuai and Li, Jianxin and Xiong, Hui and Zhang, Wancai},
  booktitle={Proceedings of the AAAI conference on artificial intelligence},
  volume={35},
  number={12},
  pages={11106--11115},
  year={2021}
}

@article{wang2024timemixer,
  title={Timemixer: Decomposable multiscale mixing for time series forecasting},
  author={Wang, Shiyu and Wu, Haixu and Shi, Xiaoming and Hu, Tengge and Luo, Huakun and Ma, Lintao and Zhang, James Y and Zhou, Jun},
  journal={arXiv preprint arXiv:2405.14616},
  year={2024}
}

@article{DBLP:journals/corr/abs-2501-13041,
  publtype={informal},
  author={Yifan Hu and Guibin Zhang and Peiyuan Liu and Disen Lan and Naiqi Li and Dawei Cheng and Tao Dai and Shu-Tao Xia and Shirui Pan},
  title={TimeFilter: Patch-Specific Spatial-Temporal Graph Filtration for Time Series Forecasting},
  year={2025},
  month={January},
  cdate={1735689600000},
  journal={CoRR},
  volume={abs/2501.13041},
  url={https://doi.org/10.48550/arXiv.2501.13041}
}
\bibliographystyle{IEEEtran}
}

\end{document}